\documentclass{article}

\usepackage[final]{neurips_2022}

\usepackage[utf8]{inputenc} 
\usepackage[T1]{fontenc}    
\usepackage{hyperref}       
\usepackage{url}            
\usepackage{booktabs}       
\usepackage{amsfonts}       
\usepackage{nicefrac}       
\usepackage{microtype}      
\usepackage{xcolor}         

\usepackage{lipsum}

\usepackage{enumitem}

\usepackage{microtype}      
\usepackage{tikz}

\usepackage{xcolor}
\usepackage{amsmath}
\usepackage{amsfonts}
\usepackage{mathtools}
\usepackage{xr-hyper}
\usepackage{cleveref}
\usepackage{hyperref}
\usepackage{comment}
\usepackage{amsthm}
\usepackage{bbm}
\usepackage{multirow}
\usepackage{graphicx}
\usepackage{subcaption}%
\usepackage{float}
\usetikzlibrary{shapes}

\definecolor{colorA}{rgb}{0.12156862745098039, 0.4666666666666667, 0.7058823529411765    }
\definecolor{colorB}{rgb}{0.8392156862745098, 0.15294117647058825, 0.1568627450980392}
\definecolor{colorC}{rgb}{0.17254901960784313, 0.6274509803921569, 0.17254901960784313}

\usepackage{titlesec}
\newcommand{\shortparagraph}[1]{\textbf{#1}\quad}

\newcommand{\vw}[1]{\textcolor{blue}{[VW: #1]}}

\DeclareRobustCommand\samplesquare[1]{%
\tikz{\node[draw,scale=0.75,regular polygon, regular polygon sides=4,fill=#1](){};}
}

\newcommand{\calF}{\mathcal{F}}
\newcommand{\bbP}{\mathbb{P}}
\newcommand{\bbR}{\mathbb{R}}
\newcommand{\bbD}{\mathbb{D}}
\newcommand{\bbN}{\mathbb{N}}
\newcommand{\bbE}{\mathbb{E}}

\newcommand{\bbQ}{\mathbb{Q}}
\newcommand{\bbC}{\mathbb{C}}
\newcommand{\calX}{\mathcal{X}}
\newcommand{\calY}{\mathcal{Y}}

\usepackage{enumitem}

\usepackage{capt-of}

\usepackage{algorithm}
\usepackage{algorithmic}

\usepackage{booktabs,siunitx}
\graphicspath{{./figures/}}

\usepackage{arydshln}

\allowdisplaybreaks[4]

\newcommand*{\mathcolor}{}
\def\mathcolor#1#{\mathcoloraux{#1}}
\newcommand*{\mathcoloraux}[3]{%
  \protect\leavevmode
  \begingroup
    \color#1{#2}#3%
  \endgroup
}

\makeatletter
\newcommand{\mypm}{\mathbin{\mathpalette\@mypm\relax}}
\newcommand{\@mypm}[2]{\ooalign{%
		\raisebox{.1\height}{$#1+$}\cr
		\smash{\raisebox{-.6\height}{$#1-$}}\cr}}
\makeatother

\newtheorem{definition}{Definition}

\newtheorem{lemma}{Lemma}
\newtheorem{theorem}[lemma]{Theorem}

\title{Generalized Variational Inference in Function Spaces:\\ Gaussian Measures meet Bayesian Deep Learning}

\author{%
  Veit D.~Wild\thanks{equal contribution, order decided by coinflip} \\
  Department of Statistics\\
  University of Oxford\\
  29 St Giles', Oxford OX1, UK \\
  \texttt{veit.wild@stats.ox.ac.uk} \\
  \And
  Robert Hu$^*$ \thanks{Work primarily done at the University of Oxford and finished at Amazon.} \\
    Amazon\\
  \texttt{robyhu@amazon.co.uk} \\
  \AND
  Dino Sejdinovic \\
  Department of Statistics \\
  University of Oxford \\
  29 St Giles', Oxford OX1, UK \\
  \texttt{dino.sejdinovic@stats.ox.ac.uk} 
}

\begin{document}

\maketitle

\begin{abstract}
We develop a framework for generalized variational inference in infinite-dimensional function spaces and use it to construct a method termed Gaussian Wasserstein inference (GWI). GWI leverages the Wasserstein distance between Gaussian measures on the Hilbert space of square-integrable functions in order to determine a variational posterior using a tractable optimization criterion. It avoids pathologies arising in standard variational function space inference. An exciting application of GWI is the ability to use deep neural networks in the variational parametrization of GWI, combining their superior predictive performance with the principled uncertainty quantification analogous to that of Gaussian processes. The proposed method obtains state-of-the-art performance on several benchmark datasets.
\end{abstract}
\vspace{-0.3cm}
\section{Introduction}
\vspace{-0.3cm}

In the past decade, considerable effort has been invested in developing Bayesian deep learning approaches \citep{welling2011bayesian,chen2014stochastic,blundell2015weight,gal2016dropout,kendall2017uncertainties,ritter2018scalable,khan2018fast,maddox2019simple}. There are at least two key advantages to Bayesian models. Firstly, Bayesian model averaging is known to improve predictive performance \citep{komaki1996asymptotic} even in misspecified situations \citep{fushiki2005bootstrap,ramamoorthi2015posterior}. The empirical success of methods such as deep ensembles \citep{lakshminarayanan2017simple} may be interpreted as compelling evidence for this claim \citep{wilson2020bayesian}. Secondly, Bayesian models provide the user with a predictive distribution for an unseen data point. This can be naturally leveraged to quantify posterior uncertainty.

Even though impressive progress has been made, there are problems that remain unresolved. The prior distribution for the unknown function is typically induced by a prior distribution over deep neural network weights (and biases). It is hard to interpret the inductive bias in a function space that is induced by such priors for weights and unclear how one might incorporate prior knowledge about the unknown function. Additionally, the resulting inference problem is extremely high-dimensional and requires approximation techniques that are either computationally expensive \citep{neal2012bayesian} or so crude that the approximate posterior may suffer from pathological behavior \citep{foong2020expressiveness}. The difficulties of performing Bayesian inference for weights have led to the emergence of methods that approach the problem in function space directly \citep{ma2019variational,sun2019functional,rudner2020rethinking,ma2021functional}. 

The theory of constructing prior distributions in function spaces is well developed and the most famous class of prior distributions are \textit{Gaussian processes}. They have been commonly used for decades in the machine learning community to elicit interpretable functional priors and are known to have well-calibrated predictive uncertainties \citep{rasmussen2003gaussian}.

In a separate thread of research, a new powerful inference framework called \textit{Generalized Variational Inference} (GVI) has been recently developed \citep{knoblauch2019generalized}. The authors argue that standard assumptions of Bayesian inference such as well-specified priors, well-specified likelihoods and infinite computing power are often violated in practice. They therefore propose a generalized view on Bayesian inference that takes these points into consideration. We extend the work of \citet{knoblauch2019generalized} to situations where no probability density functions for the prior exist and are thus able to use generalized variational inference in infinite-dimensional function spaces directly. We then specify both the prior and variational measures as Gaussian measures and measure their dissimilarity using the Wasserstein distance. This results in the method which we call \textit{Gaussian Wasserstein Inference in Function Spaces} (GWI-FS). 
An exciting application of our method is the ability to equip deep neural networks with uncertainty quantification using the framework analogous to that of Gaussian processes, resulting in a state-of-the-art method termed \emph{GWI-net}. Our main contributions are:
\vspace{-0.2cm}
\begin{itemize}[noitemsep]
    \item We create a general framework for inference in function space based on Gaussian measures on the space of square-integrable functions,
    \item We derive an objective function that can be expressed in terms of the \textit{parameters of the Gaussian measures},
    \item We derive a tractable approximation to our objective function that is valid for (almost) arbitrary kernels and mean functions,
    \item We demonstrate the utility of our method by obtaining state-of-the-art results on the UCI regression datasets and on Fashion MNIST and CIFAR 10\footnote{Codebase: \hyperref[https://anonymous.4open.science/r/GWI-D7CA/]{https://anonymous.4open.science/r/GWI-D7CA/}}.
\end{itemize}

\vspace{-0.3cm}
\section{Related Work}
\vspace{-0.3cm}

GWI-FS draws on the work developed in the Gaussian process literature, but can be used to equip traditional neural network architectures with uncertainty. We therefore give a brief overview of the relevant related methods in both the Bayesian neural network (BNNs) and Gaussian process community. 

\shortparagraph{Bayesian neural networks} Traditionally Bayesian neural networks have been assigned priors in weight space. The effects of various priors on inference and uncertainty quantification are still not well understood \citep{fortuin2021bayesian}. As the posterior (over weights) is intractable, sampling algorithms such as Hamiltonian Monte Carlo (HMC) were initially proposed \citet{neal2012bayesian}. Due to the unfavorable scaling properties of standard HMC which requires the full gradient, batch-size approximations of HMC evolved \citep{chen2014stochastic}. Another line of research exploits Langevin dynamics to generate posterior samples \citep{welling2011bayesian} in weight space.

\shortparagraph{Variational methods for BNNs in weight space} In variational inference, the true posterior is approximated by a more tractable so-called \textit{variational} distribution. The user specifies a class of approximate posterior measures and selects the best posterior approximation by maximizing the so-called evidence lower bound (ELBO). The Bayes by Backprop \citep{blundell2015weight} method is one such variational mean-field approximation of the weight-space posterior. In variational dropout \citep{gal2016dropout}, a specific approximation is chosen to reinterpret dropout \citep{srivastava2014dropout} at test time as a variational procedure.

\shortparagraph{Variational methods for BNNs in function spaces}Inference in weight space is challenging, as the problem is typically high-dimensional and the posterior distribution over weights multi-modal. This led to a line of research in which inference algorithms are formulated in function spaces. Variational implicit processes \citep{ma2019variational} approximate the BNN posterior as a linear combination of draws from the prior. Functional-BNN \citep{sun2019functional} matches a BNN to a functional prior (for example a GP) and performs inference by optimising a functional Kullback-Leibler (KL) divergence exploiting score function estimators \citep{li2017gradient,shi2018spectral}. \citet{rudner2020rethinking} use a local approximation to the prior and variational posterior processes to obtain a tractable functional Kullback-Leibler divergence. \citet{ma2021functional} generalise the variational family in \citet{ma2019variational} and obtain a more scalable procedure by using a different approximation to the functional KL-divergence. Recent work has also proposed to adapt BNN priors to interpretable functional priors by minimizing the Wasserstein distance between a BNN prior and a Gaussian process \citep{tran2020all}. Another line of research exploits the Wasserstein gradient flow and tries to encourage diversity in the function space \citep{d2021stein,d2021repulsive}.

\shortparagraph{Gaussian processes} Standard Gaussian process regression \citep{rasmussen2003gaussian} allows interpretable prior specification but scales poorly with respect to the number of data points. As a result, a plethora of approximation techniques are introduced. On one hand, there are variational approximations to the true posterior \citep{titsias2009variational,hensman2013gaussian} and several extensions \citep{hensman2017variational,salimbeni2018orthogonally,dutordoir2020sparse}. On the other hand, GPU utilization is combined with Krylov subspace methods to obtain scalability \citep{gardner2018gpytorch,wang2019exact}.

\vspace{-0.3cm}
\section{Background}
\vspace{-0.3cm}

In this section we give some background on generalized variational inference in infinite dimensions and introduce Gaussian measures in Hilbert spaces. We further discuss their relation to the more familiar Gaussian processes at the end.

\vspace{-0.2cm}
\subsection{Generalized Variational Inference in Function Spaces}\label{subsec:GVI_FSI}
\vspace{-0.2cm}

In functional variational inference, we assign a prior $p(f)$ to the unknown function $f \in E$, where $E$ is a function space\footnote{We assume $E$ to be a Polish space, which avoids technical difficulties in defining the posterior measure \citep[Chapter 1.3 ]{ghosal2017fundamentals}}. The prior is combined with the likelihood $p(y|f)$ to give the posterior $p(f|y)$. The posterior is often intractable which is why in variational inference we specify a tractable variational approximation $q(f)$ to $p(f|y)$ and train our model by maximising the evidence lower bound (ELBO) 
\begin{align}
    \mathcal{L}= \bbE_{q(f)} \big[ \log p(y|f) \big] -  \bbD_{\text{KL}} \big( q(f) , p(f) \big),
\end{align}
where $\bbD_{\text{KL}}$ denotes the KL divergence. Note that in the case where $E$ is infinite dimensional $p(f)$ and $q(f)$ cannot be probability density functions with respect to the Lebesgue measure \citep[see e.g.][for a discussion]{hunt1992prevalence}, which is why the above notation, although commonly used, is imprecise. What we in fact mean are the probability measures over $E$ associated with the prior and variational approximation. We will denote these measures as $\bbP^F$ and $\bbQ^F$ from now on to make this difference explicit. The ELBO in this notation reads as
\begin{align}\label{eq:ELBO}
    \mathcal{L}:= \bbE_{\bbQ} \big[ \log p(y|F) \big] - \bbD_{\text{KL}} \big( \bbQ^F , \bbP^F \big).
\end{align}
Note that the KL divergence (for measures) is defined as 
\begin{equation}\label{eq:def_KLD}
    \bbD_{\text{KL}} \big( \bbQ^F , \bbP^F \big) = \int \log \left( \frac{d \bbQ^F}{d\bbP^F}(f) \right) \, d\bbQ^F(f),
\end{equation}
where we assume that $\bbQ^F$ is dominated by the measure $\bbP^F$ which guarantees the existence of the Radon-Nikodym derivative $d\bbQ^F/d\bbP^F$. A number of papers focus on obtaining tractable approximations of \eqref{eq:def_KLD} \citep{sun2019functional,rudner2020rethinking,ma2021functional}. However, the use of KL-divergence in infinite-dimensional function spaces can be a delicate task, since benign constructions of priors and variational approximations may not satisfy that $\bbQ^F$ is dominated by $\bbP^F$ which leads to $\bbD_{\text{KL}} \big( \bbQ^F , \bbP^F \big) = \infty$ \citep{burt2020understanding}. This often renders the objective \eqref{eq:ELBO} useless or at least problematic.

A \textit{true Bayesian} is committed to the use of the KL divergence in \eqref{eq:ELBO} as maximizing $\mathcal{L}$ is equivalent to minimizing the KL divergence between the true posterior measure and the variational measure. This equivalence is typically demonstrated using pdfs but the argument generalizes to infinite dimensions as is shown for GPs in \citet{matthews2016sparse} or in a more measure theoretic formulation in Theorem 4 of \citet{wild2021variational}. 

However, \citet{knoblauch2019generalized} argue that given the problems of prior and likelihood specification as well as available compute, an axiomatically justified way of moving from prior to posterior beliefs is by solving a more general optimization problem \citep[Theorem 15]{knoblauch2019generalized}. Crucially it is valid to replace the KL-divergence by an arbitrary measure of dissimilarity $\mathbb{D}$ satisfying $\bbD(\bbQ^F, \bbP^F) \ge 0 $ and $\bbD(\bbQ^F, \bbP^F)=0 \Rightarrow \bbQ^F = \bbP^F$. The arguments in \citet{knoblauch2019generalized} are made assuming the existence of a pdf for the prior, but they rely solely on a reformulation of Bayesian inference as optimization problem \citep[Chapter 2]{knoblauch2019generalized}. We show in Appendix \ref{ap:Bayes_as_Opt} that this reformulation can also be made for infinite-dimensional prior measures and therefore consider the generalized loss  
\begin{equation}\label{eq:Generalized_loss}
    \mathcal{L}:= -\bbE_{\bbQ} \big[ \log p(y|F) \big] + \bbD \big( \bbQ^F , \bbP^F \big),
\end{equation}
a valid optimization objective for an arbitrary dissimilarity measure $\bbD$. This is merely an infinite-dimensional version of equation (10) in \citet{knoblauch2019generalized}. We refer to inference targeting the objective  \eqref{eq:Generalized_loss} as \emph{Generalized variational inference in function space} (GVI-FS). 

Generalised variational inference can be interpreted as regularised loss minimisation lifted into the space of probability measures. The first term in  \eqref{eq:Generalized_loss} is understood as a loss which we want to minimise on average, while the second term punishes strong deviations from the prior.


The particular instance of GVI-FS that we explore is where both $\bbP^F$ and $\bbQ^F$ are Gaussian measures (on an infinite-dimensional Hilbert space)  and $\bbD$ is chosen to be the Wasserstein metric \citep{kantorovich1960mathematical}. We will refer to this setting as \textit{Gaussian Wasserstein Inference in Function Space} (GWI-FS) or more consciously as \textit{Gaussian Wasserstein Inference} (GWI)

\vspace{-0.2cm}
\subsection{Gaussian Random Elements and Gaussian Measures in Hilbert spaces} 
\vspace{-0.2cm}

In this section we introduce Gaussian random elements (GRE) and Gaussian measures in Hilbert spaces -- these concepts are somewhat technical but crucial in the construction of our method. We then describe their close relationship to the more familiar Gaussian process notions in the next section. 

Let $\big( \Omega,\mathcal{A},\bbP \big)$ be the underlying (physical) probability space and $\big( H , \langle \cdot, \cdot \rangle \big)$ be a Hilbert space. 

\shortparagraph{Gaussian random elements} A measurable function $F: \Omega \to H$ is called GRE (in $H$) if and only if  $\langle F,h \rangle: \Omega \to \bbR$ has a scalar Gaussian distribution for all $h \in H$.\footnote{We allow for the degenerate case where the variance of $\langle F, h \rangle$ is zero. This means we interpret a Gaussian with variance zero as Dirac measure.} Every GRE $F$ has a mean element $m \in H$ defined by 
\begin{equation}
    m:= \int F(\omega) \, d\bbP(\omega) 
\end{equation}
and a (linear) covariance operator $C: H \to H$ defined by
\begin{equation}
    Ch(\cdot) := \int \langle F(\omega),h \rangle F(\omega) \, \bbP(\omega) - \langle m,h \rangle m.
\end{equation}
for $h \in H$. Both integrals are to be understood as Bochner integrals \citep[Chapter 3]{kukush2020gaussian}. The Bochner integral has the property that 
$
    \big \langle \int F(\omega) \, d\bbP(\omega), h \rangle = \int \langle F(\omega), h \rangle \, d \bbP(\omega)
$
for all $h \in H$. This combined with Fubini's theorem and the definition of a GRE implies that
\begin{equation}
    \langle F, h \rangle \sim \mathcal{N}( \langle m, h\rangle , \langle Ch , h \rangle \big),
\end{equation}
for any $h \in H$ with $\mathcal{N}(\mu, \sigma^2)$ denoting the normal distribution with mean $\mu \in \bbR$ and variance $\sigma^2 > 0$. Similarly we denote $F \sim \mathcal{N}(m, C)$ for a GRE in $H$ with mean element $m$ and covariance operator $C$. It can be shown that the covariance operator $C$ of a GRE is a positive self-adjoint trace-class operator. Conversely, for every positive self-adjoint trace class operator and every $m \in H$, there exists a GRE with $F \sim \mathcal{N}(m,C)$ \citep[Theorem 2.3.1]{Bogachev1998}.

\shortparagraph{Gaussian measures} The push-forward measure of $\bbP$ through $F$ is defined as $\bbP^F(A):= \bbP\big( F^{-1}(A) \big)$ for all Borel-measurable $A \subset H$. If $F\sim \mathcal{N}(m,C)$ is a GRE, we call $P:= \bbP^F$ a GM and write $P=\mathcal{N}(m,C)$. Note that GMs or equivalently GREs allow us to specify probability distributions over (infinite-dimensional) Hilbert spaces by using a given mean element and a given covariance operator. 

Details about Gaussian Measures in Hilbert spaces can be found in Chapter 2 of \citet{da2014stochastic} or in \citet{kukush2020gaussian}. In fact, Gaussian measures can be defined on even more general linear spaces such as Banach or Fréchet spaces \citep{Bogachev1998}.
\vspace{-0.2cm}
\subsection{Gaussian Processes and Their Corresponding Measures}
\vspace{-0.2cm}

In this section we describe how Gaussian processes -- a standard tool to assign functional priors in Bayesian machine learning -- are related to Gaussian measures.

Let $\big( \Omega,\mathcal{A},\bbP \big)$ be the underlying (physical) probability space and $\calX \subset \bbR^D$ be measurable. The (product-) measurable mapping $G:\Omega \times \calX \to \bbR$ is called a Gaussian process (GP) if and only if for all $N \in \bbN$ and all $X=\{x_n\}_{n=1}^N \subset \calX$ the random vector $G(X):=\big( G(\cdot,x_1), \hdots, G(\cdot,x_N) \big)^T$ is multivariate Gaussian. For a GP $G$ we define a mean function $m(x):= \bbE \big[ G(x) \big]$, $x \in \calX$, and a covariance function by $k(x,x'):= \bbC[G(x), G(x') \big]$ for $x,x' \in \calX$. Here $\bbE$ denotes the expected value and $\bbC[\cdot,\cdot]$ the covariance. It follows from the definition that $G(X) \sim \mathcal{N}\big( m(X), k(X,X) \big)$ for any $\{x_n\}_{n=1}^N \subset \calX$, where we define $m(X):= \big( m(x_n) \big)_{n=1}^N$ and $k(X,X):= \big( k(x_n,x_{n'}) \big)_{n,n'=1}^N$. We write $G \sim GP(m,k)$ for a GP with mean function $m$ and covariance function $k$. Note that by the properties of the covariance we know that $k(X,X)$ is a (symmetric) positive semi-definite matrix for all $\{x_n\}_{n=1}^N \subset \calX$ and $N \in \bbN$. A function with this property is called \textit{kernel}, a terminology that we adopt henceforth. Kolmogorov's existence theorem \citep[Section 36]{billingsley2008probability} guarantees the existence of a Gaussian process for any kernel $k$ and any mean function $m$. The standard reference for Gaussian processes in machine learning is \citet{rasmussen2003gaussian}.

The main advantage of Gaussian processes in specifying priors over a function space is that the kernel $k$ allows us to incorporate readily interpretable prior assumptions, such as smoothness or periodicity. For example, choosing the squared exponential kernel \citep{rasmussen2003gaussian} implies that the unknown function is infinitely differentiable and that the correlation of the functional output is higher the closer the inputs are. 

In order to insert the Gaussian process prior into our generalized loss in \eqref{eq:Generalized_loss} we need to know the probability measure that is associated to the Gaussian process. In general, we can associate more than one Gaussian measure with a given Gaussian process. For example:
\vspace{-0.2cm}
\begin{itemize}[noitemsep]
    \item If the GP has continuous sample paths we can associate a Gaussian measure on the space $E$ of continuous functions with it \citep[Example 2.4]{lifshits2012lectures}. 
    \item If the GP has square-integrable sample paths we can associate a Gaussian measure on the Hilbert space of square-integrable functions with it (cf. Theorem \ref{thm:GP_as_GM}).
\end{itemize}
These sample path properties can be guaranteed under additional assumptions on the kernel. The next theorem discusses one such kernel condition which guarantees the GP to have sample paths in the Hilbert space of square integrable functions, denoted  $L^2(\calX,\rho,\bbR)$, with inner product  $
    \langle g, h \rangle_2 := \int_\calX g(x) h(x) \, d \rho(x).
$

\begin{theorem}\label{thm:GP_as_GM}
Let $F \sim GP(m,k)$ be a GP with mean $m \in L^2(\calX,\rho,\bbR)$ and kernel $k$ such that
\begin{equation}\label{eq:kernel_cond}
  \int_{\calX} k(x,x) \, d \rho(x) < \infty .
\end{equation}
We call a kernel satisfying \eqref{eq:kernel_cond} trace-class kernel. Then the mapping $\widetilde{F}: \Omega \to L^2(\calX,\rho,\bbR)$ defined as $\widetilde{F}(\omega):=F(\omega,\cdot)$ is a Gaussian random element with mean $m$ and covariance operator C given as 
\begin{equation}\label{eq:cov_op}
    Cg(\cdot):= \int k(\cdot,x')g(x') \, d \rho(x')
\end{equation}
for any $g \in L^2(\calX,\rho,\bbR)$. Consequently $P:=\bbP^F \sim \mathcal{N}(m,C)$ is a Gaussian measure.
\end{theorem}
\begin{proof}
The fact that $\widetilde{F}$ as defined above is a GRE follows immediately from Example 2.3.16 in \citet{Bogachev1998}. The fact that $m$ is its mean and $C$ as defined in \eqref{eq:cov_op} is its covariance operator follows from Fubini's theorem.
\end{proof}
It shall be noted that there is no need to appeal to GPs in order to justify the use of GMs. In fact, it has recently been demonstrated that variational inference for GPs can be formulated purely in terms of GMs \citep{wild2021variational}. In the following sections we will therefore deploy GMs without any reference to GPs, but it is of course always possible to think of them as the measures that correspond to GPs where the kernel satisfies an additional assumption such as \eqref{eq:kernel_cond}.

\vspace{-0.3cm}
\section{Gaussian Wasserstein Inference in Function Spaces}
\vspace{-0.3cm}

This section describes how the Wasserstein distance between Gaussian measures can be used to obtain a tractable optimization target for inference in function spaces. In the end, we discuss several parametrizations of GWI and introduce our main inference method - the GWI-net.
\vspace{-0.2cm}
\subsection{Model description}
\vspace{-0.2cm}

Let $\{(x_n,y_n)\}_{n=1}^N \subset \calX \times \calY$ be $N \in \bbN$ paired observations. We assume that $\calX \subset \bbR^D$, $D \in \bbN$ and further that $\mathcal{Y}=\bbR$ for regression and $\mathcal{Y}=\{1,\hdots,J\}$ for classification with $J \in \bbN$ classes. We focus in our exposition here on the regression case but have given the relevant derivations for classification in Appendix \ref{ap:Classification}.

As pointed out in section \ref{subsec:GVI_FSI}, GVI in function space minimises the generalized loss $ \mathcal{L}= -\bbE_{\bbQ} \big[ \log p(y|F) \big] + \bbD \big( \bbQ^F , \bbP^F \big)$.
We make the mild assumption that the unknown function $f$ is square integrable with respect to the data distribution $\rho$ on $\calX$ which means $f \in E= L^2(\calX,\rho,\bbR)$. The prior $P:=\bbP^{F}$ is described by a Gaussian measure with mean $m_P \in \mathcal{L}^2(\calX,\rho, \bbR)$ and covariance operator $C_P$ described by a trace-class kernel $k:\calX \times \calX \to \bbR$ which means it is given as
$
    (C_Pf)(x) := \int_\calX k(x, x') f(x') \, d\rho(x')
$
for all $f \in L^2(\calX, \rho, \bbR)$. We assume a Gaussian likelihood for $y:=(y_1,\hdots,y_N)$ given as $p(y|f):= \prod_{n=1}^N p(y_n|f)$\footnote{Astute readers may notice that the definition of the likelihood contains a pointwise evaluation $f(x_n)$ which may not be a well defined operation on $L^2(\calX,\rho,\bbR)$. We detail in Appendix \ref{ap:eq:ptw_ev} how that problem can be circumvented and that in fact $F(x) \sim \mathcal{N}(m(x),k(x,x))$ as one would expected.} with 
\begin{equation}
    p(y_n|f) := \mathcal{N}(y_n \, | \, f(x_n), \sigma^2),
\end{equation}
where $\mathcal{N}(\cdot \, | \, \mu, \sigma^2) $ denotes the pdf of a normal distribution with mean $\mu \in \bbR$ and variance $\sigma^2>0$. This prior and likelihood are natural choices as they mimic the standard formulation of Gaussian process regression. The variational approximation of the posterior is chosen to be another Gaussian measure $Q:=\bbQ^F$ with arbitrary mean $m_Q \in L^2(\calX, \rho, \bbR) $ and arbitrary covariance operator $C_Q$ induced by a trace-class kernel $r$:
$
    (C_Qf)(x) := \int_\calX r(x, x') f(x') \, d\rho(x')
$
for all $f \in L^2(\calX, \rho, \bbR)$.

It remains for us to select a dissimilarity measure $\bbD$. As already pointed out in the introduction we decide to use the Wasserstein distance $W_2$ (a formal definition is given in Appendix \ref{ap:WD}). This choice was guided by two considerations:
\begin{enumerate}
    \item The Wasserstein metric was proven to be a useful metric for probability distributions in machine learning applications \citep{arjovsky2017wasserstein,tran2020all}. Furthermore the Wasserstein metric is known to have desirable statistical properties \citep{panaretos2019statistical}.
    \item The Wasserstein distance is tractable for arbitrary Gaussian measures on (separable) Hilbert spaces \citep{gelbrich1990formula} and given as
    \begin{equation}\label{eq:WD_Gaussians}
    W_2^2(P,Q) = \| m_P - m_Q \|_2^2 + tr(C_P) + tr(C_Q) - 2 \cdot tr \Big[ \big(C_P^{1/2} C_Q^{} C_P^{1/2} \big)^{1/2} \Big],
    \end{equation}
    where $tr$ denotes the trace of an operator and $C_P^{1/2}$ is the square root of the positive, self-adjoint operator $C_P$. This is in stark contrast to the KL-divergence that is infinite whenever $\bbQ^F$ is not dominated by $\bbP^F$ and even in the case where it is finite there exists no explicit formula for the KL-divergence in infinite dimensions.
\end{enumerate}

The generalized loss for our model is therefore given as 
\begin{align}
    \mathcal{L}= - \sum_{n=1}^N \bbE_{\bbQ} \Big[ \log \mathcal{N}\big(y_n \, | \, F(x_n), \sigma^2\big) \Big] + W_2(P,Q). \label{eq:GL_exact}
\end{align}
Note that the expected log-likelihood in \eqref{eq:GL_exact} can be calculated analytically as
\begin{equation}\label{eq:ELL}
\bbE_{\bbQ} \Big[ \log \mathcal{N}\big(y_n \, | \, F(x_n), \sigma^2\big) \Big] = -\frac{N}{2} \log( 2 \pi \sigma^2) -  \sum_{n=1}^N \frac{\big( y_n - m_Q(x_n) \big)^2 +r(x_n,x_n)}{2 \sigma^2} .
\end{equation}
It remains to produce an approximation of \eqref{eq:WD_Gaussians} in order to obtain a tractable inference procedure. To this end, note that by definition
$
    \| m_P - m_Q |_2^2 = \int \big ( m_P(x) - m_Q(x) \big)^2 \, d \rho(x)
$
and further 
$
    tr(C_P) = \int k(x,x) \, d \rho(x)
$ \citep{brislawn1991traceable}.
We now replace the true input distribution $\rho$ with the empirical data distribution $\widehat{\rho}:= \frac{1}{N} \sum_{n=1}^N \delta_{x_n}$, where $\delta_x$ denotes the Dirac measure in $x \in \calX$. This gives $\| m_P - m_Q \|_2^2 \approx \frac{1}{N} \sum_{n=1}^N \big( m_P(x_n) - m_Q(x_n) \big)^2$, $tr(C_P) \approx \frac{1}{N} \sum_{n=1}^N k(x_n,x_n)$ and $tr(C_Q)\approx \frac{1}{N} \sum_{n=1}^N r(x_n,x_n)$.
It remains to provide an approximation of $tr \Big[ \big(C_P^{1/2} C_Q^{} C_P^{1/2} \big)^{1/2} \Big]$. The key idea is to approximate the spectrum of $C_P^{1/2} C_Q^{} C_P^{1/2}$ by that of an appropriate kernel matrix. Details are discussed in Appendix \ref{ap:WD_approximation}. This leads to the following final approximation for the Wasserstein metric
\begin{align}
    \hat{W}^2:= &\frac{1}{N} \sum_{n=1}^N \big( m_P(x_n)-m_Q(x_n) \big)^2 + \frac{1}{N} \sum_{n=1}^N k(x_n,x_n\label{eq:WD_approx1})\\
    &+ \frac{1}{N} \sum_{n=1}^N r(x_n,x_n) - \frac{2}{\sqrt{N N_S}} \sum_{s=1}^{N_S} \sqrt{\lambda_s\big(  r(X_S,X)k(X,X_S) \big)}, \label{eq:WD_approx2}
\end{align}
where $X_S:=(x_{S,1},\hdots,x_{S,N_S}) $ with $x_{S,1},\hdots x_{S,N_S} \in \bbR^D $ being subsampled from the input data $X$. Further $r(X_S,X):= \big( r(x_{S,s},x_n) \big)_{s,n}$ and $k(X,X_S):= \big( k(x_n,x_{S,s}) \big)_{n,s}$ for $n=1,\hdots,N$, $s=1,\hdots,N_S$ and $\lambda_s\big(  r(X_S,X)k(X,X_S) \big)$ denotes the $s$-th eigenvalue of the matrix $r(X_S,X)k(X,X_S) \in \bbR^{N_S \times N_S}$. The approximation quality of $\widehat{W}$ is related to the spectral decay of the operator $C_P C_Q$, which in turn is determined by the kernels $k$ and $r$. For the choices made in Section \ref{sec:paramterisation} we empirically observe rapid spectral decay (cp. Appendix \ref{Ap:wasserstein_estimation error}) and therefore are confident that the 2-Wasserstein distance is estimated reliably for our method.


The combination of \eqref{eq:ELL}, \eqref{eq:WD_approx1} and \eqref{eq:WD_approx2} gives a generalized loss that is tractable in terms of $m_P,m_Q,k$, and $r$. If we disregard computation time of $m_P,m_Q,k$ and $r$, the generalized loss can be evaluated in $\mathcal{O}(N + N_S^2 N+ N_S^3) $, where typically $N_S \ll N$, e.g. $N_S=100$. We provide a batch version of our loss in Appendix \ref{ap:Batch_Mode} which reduces the computations to $\mathcal{O}( N_S^2 N_B + N_S^3)$ where $N_B \ll N$ is the batch-size. Note, however, that the final computation time for our method will be determined by the complexity hidden in the evaluation of $m_Q$, $m_P$, $k$, and $r$ as we need $N_B$ evaluations of $m_Q$ and $m_P$ and $N_S \cdot N_B$ evaluations of $r$ and $k$ per iteration.

\vspace{-0.2cm}

\subsection{Parameterisations of Prior and Variational Measure}\label{sec:paramterisation}
\vspace{-0.2cm}

The prior for our model is given as $P= \mathcal{N}(m_P,C_P)$ with $C_P$ induced by a trace-class kernel $k$. One of the advantages of the proposed approach is that any trace-class kernel is allowed and this is where one can incorporate specific assumptions and domain expertise. This is a thoroughly studied topic: the prior kernel can encode periodicity \citep{durrande2016detecting}, geometric intuition \citep{van2018learning}, and even model linear constraints for the unknown function \citep{jidling2017linearly}. In order to keep the exposition simple and maintain focus on the inference, however, and in line with using simple priors on network weights in standard Bayesian deep learning, we opt for a simple zero mean prior $m_P=0$ and a standard ARD kernel $k$ given as
\begin{equation}\label{eq:ARD_kernel}
 k(x,x') = \sigma_f^2 \exp \Big( - \frac{1}{2} \sum_{d=1}^D  \frac{(x_d - x'_d)^2}{\alpha_d^2}  \Big)  \end{equation}
for $x,x' \in \calX \subset \bbR^D$. We refer to $\sigma_f > 0$ as \textit{kernel scaling factor} and to $\alpha_d >0$ as \textit{length-scale} for dimension $d$. The parameters $\sigma_f$ and $\alpha:=(\alpha_1,\hdots,\alpha_D)$ are called \textit{prior hyperparameters}.

The rest of the section explores various choices for the variational mean $m_Q$ and the variational kernel $r$. The parameters appearing in the specification of $m_Q$ and $r$ are referred to as \textit{variational parameters}.

\shortparagraph{GWI: Stochastic variational Gaussian process} Let $z_1, \hdots, z_M \in \calX$ be a subsample of the data $X$ with $M \ll N$. We define the posterior mean 
\begin{equation}
    m_Q(x):= m_P(x) + \sum_{m=1}^M  \beta_m k_m(x) \label{eq:svgp_mean}
\end{equation}
with $\beta_m \in \bbR$ and $k_m(x):= k(x,z_m)$, $m=1,\hdots,M$ where $k$ is the prior kernel $k$ and $\beta:=(\beta_1,\hdots,\beta_M) \in \bbR^M$ are variational parameters. Define further the variational kernel
\begin{equation}\label{eq:r_SVGP}
        r(x,x') = k(x,x') - k_Z(x)^T k(Z,Z)^{-1} k_Z(x) +  k_Z(x)^T \Sigma k_Z(x),
\end{equation}
where $\Sigma \in \bbR^{M \times M}$ is the symmetric and positive definite variational covariance matrix that parameterises $r$.
This choice of $m_Q$ and $r$ essentially recovers the \textit{stochastic variational Gaussian processes} (SVGP)
model of \citet{titsias2009variational}. Note that in our framework it is straightforward to use all (or just more) basis functions for the mean $ m_Q(x):= m_P(x) + \sum_{n=1}^N \beta_n k_n(x)$ where $k_n(x):= k(x,x_n)$, $\beta_n \in \bbR$, $n=1,\hdots,N$. This mirrors the construction in \citet{cheng2017variational} where we allow more parameters to learn the mean than in SVGP. However, both \citet{titsias2009variational} and \citet{cheng2017variational} use a different objective function than GWI to learn the unknown parameters.
 
\shortparagraph{GWI: deep neural network with SVGP} An interesting approach is to parameterise the posterior mean as a deep neural network (DNN). We assume the DNN has $L \in \bbN$ hidden layers and the width of layer $\ell=1,\hdots,L$ is denoted $D_\ell$ with $D_0:=D$ and $D_{L+1}=1$. This means we define $g^{1}(x):= W^1 x + b^{1}$ and further
$
    h^{\ell}(x) := \phi \big( g^{\ell}(x) \big), \, 
    g^{\ell+1}(x):= W^{\ell+1}  h^{\ell}(x) + b^{\ell+1}
$
for $\ell=1,\hdots,L$. Here $W^{l+1}$ is $D_{\ell+1} \times D_\ell$ matrix, $b^{\ell+1} \in \bbR^{D_{\ell+1}}$ is a bias vector for layer $l$ and $\phi$ an activation function. We can then define the variational mean as $m_Q(x):= m_P(x) + g^{L+1}(x) $. If we choose the SVGP kernel $r$ in \eqref{eq:r_SVGP}, we essentially predict with a neural network and quantify uncertainty with a (sparse) Gaussian process, capturing the beneficial properties of both.

Neural networks have been combined in several ways with GPs \citep{wilson2016deep,tran2020all}. However, to the best of our knowledge they were not used to directly parametrize the posterior in the context of generalized variational inference in function space. The spirit of our approach is fundamentally different: rather than thinking of a neural network as a model which needs to be made Bayesian, we use it as a parametrisation of a variational posterior.

We note that we do not here provide an exhaustive study on how to best parameterize the variational measure. This paper is focused on demonstrating the ability of the proposed method to obtain valid uncertainty quantification. An exploratory study on how properties and quality of uncertainty quantification relate to different choices of $m_Q$ and $r$ is reserved for future work. We mention potential problems that can occur from misspecification in Appendix \ref{ap:model_misspecification}.
\vspace{-0.3cm}
\section{Experiments}
\vspace{-0.3cm}

We show results for GWI with the SVGP mean \eqref{eq:svgp_mean} and the SVGP kernel \eqref{eq:r_SVGP}. We use the shorthand GWI: SVGP for this approach. Additionally we implement the DNN mean with the SVGP kernel \eqref{eq:r_SVGP}. This combination achieves impressive results on various regression and classification tasks. We call this method GWI: DNN-SVGP or simply GWI-net.

\shortparagraph{Illustrative Examples} In \Cref{1d_illustration} we illustrate GWI-net on a few toy examples. One can clearly see that the posterior predictive variance expands for regions lacking observations which demonstrates the ability of our method to quantify uncertainty. We provide an additional graphic comparison with SVGP in Appendix \ref{ap:addtional_plots} and an example for two-dimensional inputs in Appendix \ref{ap:Illustrative_Example}


There we show that the pathologies regarding the quantification of in-between uncertainty discussed in \citet{foong2020expressiveness} are not present for our method.
\begin{figure}
    \centering
    \includegraphics[width=0.33\linewidth]{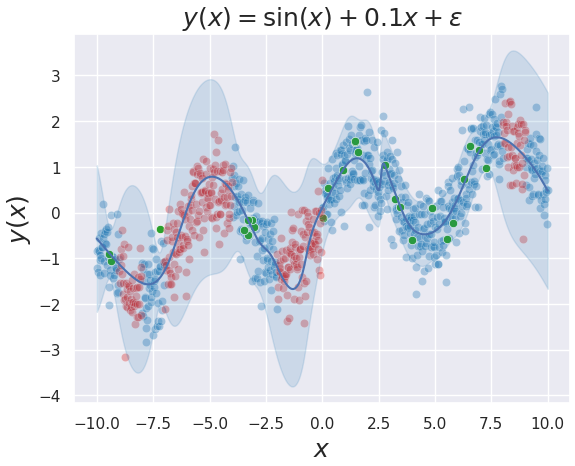}%
        \includegraphics[width=0.33\linewidth]{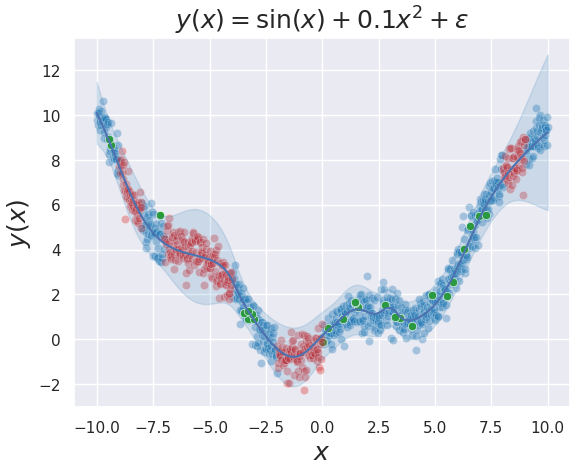}%
            \includegraphics[width=0.33\linewidth]{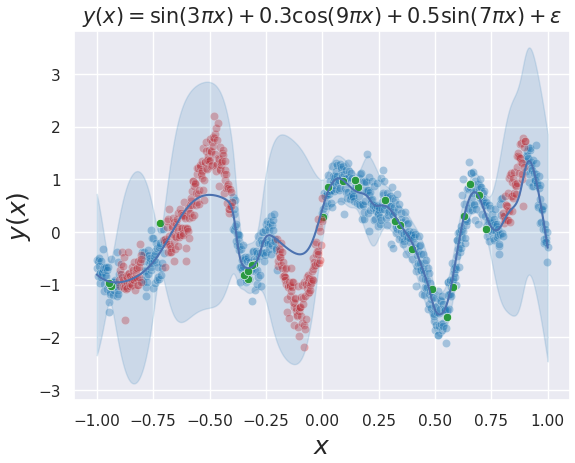}
    \caption{{\color{colorA} \samplesquare{}}: Training data $\quad$ {\color{colorB} \samplesquare{}}: Unseen data $\quad$ {\color{colorC} \samplesquare{}}: Inducing points  \\
    We query the above functions at $N=1000$ equidistant points and add white noise with $\epsilon \sim \mathcal{N}(0,0.5^2)$. We use $M=30$ inducing points and train our method as described in Appendix \ref{ap:Implementation}. The plot shows $m_Q(x)\pm1.96\sqrt{\mathbb{V}[Y^*(x)|Y]}$ where $\mathbb{V}[Y^*(x)|Y]$ is the posterior predictive variance given as $r(x,x)+\sigma^2$.  }
    \label{1d_illustration}
\end{figure}

\shortparagraph{UCI Regression} In Table \ref{tab:UCI} we report the average test negative log-likelihood (NLL) (cf. Appendix \ref{ap:Implementation} for details) of GWI: SVGP and GWI-net (GWI: DNN-SVGP) and the results of several weight-space approaches for BNNs: Bayes-by-Backprop (BBB) \citep{blundell2015weight}, variational dropout (VDO) \citep{gal2016dropout}, and variational alpha dropout ($\alpha = 0.5$) \citep{li2017dropout}. We also compare with
four function-space BNN inference methods: functional variational inference  with BNN prior (FVI) \citep{ma2021functional}, variationally implicit processes (VIP) with BNNs, VIP-Neural processes \citep{ma2019variational}, and functional BNNs (FBNNs) \citep{sun2019functional}. In order to ensure a fair comparison we matched neural network architectures and training procedures for the different methods. Detailed explanations are given in Appendix \ref{ap:Implementation}.


\begin{table}[hbt!]
\centering
 \resizebox{\linewidth}{!}{%
\begin{tabular}{lrrllrrrrrrl|rlrl}
\hline
\multicolumn{1}{c}{\multirow{2}{*}{Dataset}} & \multicolumn{1}{c}{\multirow{2}{*}{N}} & \multicolumn{1}{c}{\multirow{2}{*}{D}} & \multicolumn{2}{c}{GWI}                                      & \multicolumn{1}{c}{\multirow{2}{*}{FVI}} & \multicolumn{1}{c}{\multirow{2}{*}{VIP-BNN}} & \multicolumn{1}{c}{\multirow{2}{*}{VIP-NP}} & \multicolumn{1}{c}{\multirow{2}{*}{BBB}} & \multicolumn{1}{c}{\multirow{2}{*}{VDO}} & \multicolumn{2}{c|}{\multirow{2}{*}{$\alpha$ = 0.5}} & \multicolumn{2}{c}{\multirow{2}{*}{FBNN}} & \multicolumn{2}{c}{\multirow{2}{*}{\textbf{EXACT GP}}} \\
\multicolumn{1}{c}{}                         & \multicolumn{1}{c}{}                   & \multicolumn{1}{c}{}                   & \multicolumn{1}{c}{SVGP} & \multicolumn{1}{c}{DNN-SVGP}      & \multicolumn{1}{c}{}                     & \multicolumn{1}{c}{}                         & \multicolumn{1}{c}{}                        & \multicolumn{1}{c}{}                     & \multicolumn{1}{c}{}                     & \multicolumn{2}{c|}{}               & \multicolumn{2}{c}{}          & \multicolumn{2}{c}{}                 \\ \hline
BOSTON                                       & 506                                    & 13                                     & 2.8$\pm$0.31             & \textbf{2.27$\pm$0.06}            & 2.33$\pm$0.04                            & 2.45$\pm$0.04                                & 2.45$\pm$0.03                               & 2.76$\pm$0.04                            & 2.63$\pm$0.10                            & \multicolumn{2}{r|}{2.45$\pm$0.02}                   & \multicolumn{2}{r}{2.30$\pm$0.10}         & \multicolumn{2}{r}{2.46$\pm$0.04}                      \\
CONCRETE                                     & 1030                                   & 8                                      & 3.24$\pm$0.09            & \textbf{2.64$\pm$0.06}            & 2.88$\pm$0.06                            & 3.02$\pm$0.02                                & 3.13$\pm$0.02                               & 3.28$\pm$0.01                            & 3.23$\pm$0.01                            & \multicolumn{2}{r|}{3.06$\pm$0.03}                   & \multicolumn{2}{r}{3.09$\pm$0.01}         & \multicolumn{2}{r}{3.05$\pm$0.02}                      \\
ENERGY                                       & 768                                    & 8                                      & 1.81$\pm$0.19            & 0.91$\pm$0.12                     & 0.58$\pm$0.05                            & \textbf{0.56$\pm$0.04}                       & 0.60$\pm$0.03                               & 2.17$\pm$0.02                            & 1.13$\pm$0.02                            & \multicolumn{2}{r|}{0.95$\pm$0.09}                   & \multicolumn{2}{r}{0.68$\pm$0.02}         & \multicolumn{2}{r}{0.54$\pm$0.02}                      \\
KIN8NM                                       & 8192                                   & 8                                      & -0.86$\pm$0.38           & \textbf{-1.2$\pm$0.03}            & -1.15$\pm$0.01                           & -1.12$\pm$0.01                               & -1.05$\pm$0.00                              & -0.81$\pm$0.01                           & -0.83$\pm$0.01                           & \multicolumn{2}{r|}{-0.92$\pm$0.02}                  & \multicolumn{2}{l}{N/A$\pm$0.00}          & \multicolumn{2}{l}{N/A$\pm$0.00}                       \\
POWER                                        & 9568                                   & 4                                      & 3.35$\pm$0.22            & 2.74$\pm$0.02                     & \textbf{2.69$\pm$0.00}                   & 2.92$\pm$0.00                                & 2.90$\pm$0.00                               & 2.83$\pm$0.01                            & 2.88$\pm$0.00                            & \multicolumn{2}{r|}{2.81$\pm$0.00}                   & \multicolumn{2}{l}{N/A$\pm$0.00}          & \multicolumn{2}{l}{N/A$\pm$0.00}                       \\
PROTEIN                                      & 45730                                  & 9                                      & \textbf{2.84$\pm$0.04}   & 2.87$\pm$0.0                      & 2.85$\pm$0.00                            & 2.87$\pm$0.00                                & 2.96$\pm$0.02                               & 3.00$\pm$0.00                            & 2.99$\pm$0.00                            & \multicolumn{2}{r|}{2.90$\pm$0.00}                   & \multicolumn{2}{l}{N/A$\pm$0.00}          & \multicolumn{2}{l}{N/A$\pm$0.00}                       \\
RED WINE                                     & 1588                                   & 11                                     & 0.97$\pm$0.02            & \textbf{0.76$\pm$0.08}            & 0.97$\pm$0.06                            & 0.97$\pm$0.02                                & 1.20$\pm$0.04                               & 1.01$\pm$0.02                            & 0.97$\pm$0.02                            & \multicolumn{2}{r|}{1.01$\pm$0.02}                   & \multicolumn{2}{r}{1.04$\pm$0.01}         & \multicolumn{2}{r}{0.26$\pm$0.03}                      \\
YACHT                                        & 308                                    & 6                                      & 2.37$\pm$0.55            & 0.29$\pm$0.1                      & 0.59$\pm$0.11                            & \textbf{-0.02$\pm$0.07}                      & 0.59$\pm$0.13                               & 1.11$\pm$0.04                            & 1.22$\pm$0.18                            & \multicolumn{2}{r|}{0.79$\pm$0.11}                   & \multicolumn{2}{r}{1.03$\pm$0.03}         & \multicolumn{2}{r}{0.10$\pm$0.05}                      \\
NAVAL                                        & 11934                                  & 16                                     & \textbf{-7.25$\pm$0.08}  & -6.76$\pm$0.1                     & -7.21$\pm$0.06                           & -5.62$\pm$0.04                               & -4.11$\pm$0.00                              & -2.80$\pm$0.00                           & -2.80$\pm$0.00                           & \multicolumn{2}{r|}{-2.97$\pm$0.14}                  & \multicolumn{2}{r}{-7.13$\pm$0.02}        & \multicolumn{2}{l}{N/A$\pm$0.00}                       \\ \hdashline
Mean Rank                                    & \multicolumn{1}{c}{}                   & \multicolumn{1}{c}{}                   & \multicolumn{1}{c}{5.5}  & \multicolumn{1}{c}{\textbf{2.06}} & \multicolumn{1}{c}{2.22}                 & \multicolumn{1}{c}{3.33}                     & \multicolumn{1}{c}{4.94}                    & \multicolumn{1}{c}{7}                    & \multicolumn{1}{c}{6.11}                 & \multicolumn{2}{c|}{4.83}                            & \multicolumn{2}{c}{}                      & \multicolumn{2}{c}{}                                   \\ \hline
\end{tabular}
}
\caption{The table shows the average test NLL on several UCI regression datasets. We train on random $90\%$ of the data and predict on $10\%$. This is repeated 10 times and we report mean and standard deviation. The results for our competitors are taken from \citet{ma2021functional}. }
\label{results}
\label{tab:UCI}
\vspace{-0.75cm}
\end{table}

One can see that GWI-net obtains the best mean rank of all methods being the best model on 4/9 datasets and performing competitively on all datasets. Note that we exclude FBNN and exact Gaussian processes from the comparison because their computational complexity is often prohibitively large.

\shortparagraph{Classification and OOD Detection}\label{sec:exp_classification}
We demonstrate the ability of GWI to perform image classifications on Fashion MNIST \citep{xiao2017fashion} and CIFAR-10 \citep{krizhevsky2009learning}. We compare to FVI, mean-field variational inference (MVFI) \citep{blundell2015weight}, maximum a posteriori approximation (MAP), K-FAC Laplace-GNN \citep{martens2015optimizing} and its dampened version \citep{ritter2018scalable}. Implementation details are discussed in \ref{ap:Implement_Details_Class}.


We also assess the ability of our model to perform out-of-distribution detection using in-distribution (ID) / out of-distribution (OOD) pairs given as FashionMNIST/MNIST and CIFAR10/SVNH. Following the
setting of \citet{osawa2019practical,immer2021improving} we calculate the area under the curve (AUC) of a binary out-of-distribution classifier based on predictive entropies. Results are shown in \Cref{class}.

\begin{table}[htb!]
 \resizebox{\linewidth}{!}{%

\begin{tabular}{lrrrrrr}
\hline
               & \multicolumn{3}{c}{FMNIST}                                                                                                                 & \multicolumn{3}{c}{CIFAR 10}                                                                                                      \\ \cline{2-7} 
\textbf{Model} & \multicolumn{1}{c}{Accuracy}                 & \multicolumn{1}{c}{NLL}                      & \multicolumn{1}{c}{OOD-AUC}                  & \multicolumn{1}{c}{Accuracy}                 & \multicolumn{1}{c}{NLL}                      & \multicolumn{1}{c}{OOD-AUC}         \\ \hline
GWI-net        & \multicolumn{1}{l}{\textbf{93.25 $\pm$0.09}} & \multicolumn{1}{l}{\textbf{0.250 $\pm$0.00}} & \multicolumn{1}{l}{\textbf{0.959 $\pm$0.01}} & \multicolumn{1}{l}{\textbf{83.82 $\pm$0.00}} & \multicolumn{1}{l}{\textbf{0.553 $\pm$0.00}} & \multicolumn{1}{l}{0.618 $\pm$0.00} \\
FVI            & 91.60$\pm$0.14                               & 0.254$\pm$0.05                               & 0.956$\pm$0.06                               & 77.69 $\pm$0.64                              & 0.675$\pm$0.03                               & 0.883$\pm$0.04                      \\
MFVI           & 91.20$\pm$0.10                               & 0.343$\pm$0.01                               & 0.782$\pm$0.02                               & 76.40$\pm$0.52                               & 1.372$\pm$0.02                               & 0.589$\pm$0.01                      \\
MAP            & 91.39$\pm$0.11                               & 0.258$\pm$0.00                               & 0.864$\pm$0.00                               & 77.41$\pm$0.06                               & 0.690$\pm$0.00                               & 0.809$\pm$0.01                      \\
KFAC-LAPLACE   & 84.42$\pm$0.12                               & 0.942$\pm$0.01                               & 0.945$\pm$0.00                               & 72.49$\pm$0.20                               & 1.274$\pm$0.01                               & 0.548$\pm$0.01                      \\
RITTER et al. & 91.20$\pm$0.07                               & 0.265$\pm$0.00                               & 0.947$\pm$0.00                               & 77.38$\pm$0.06                               & 0.661$\pm$0.00                               & 0.796$\pm$0.00                      \\ \hline
\end{tabular}
}
\caption{We report average accuracy, NLL and OOD-AUC on test data for 10 different train/test splits. The results for FVI are obtained from \citet{ma2021functional} and for MAP, KFAC and Ritter et al. results are taken from \citet{immer2021improving} .}
\label{class}
\end{table}
\vspace{-0.5cm}

Our method performs best in all categories on the Fashion MNIST dataset achieving state-of-the-art results. On CIFAR10 we obtain the highest accuracy and best NLL by a significant margin and perform competitively in the OOD detection task.
\vspace{-0.3cm}

\section{Limitations}

In this section we discuss some of the shortcomings and difficulties which are related to our method.

The GVI-FS framework allows the specification of function space inference via infinite dimensional parameters such as mean and kernel functions. This great flexibility essentially allows the specification of mismatched prior and posterior parameters. We illustrate such a case in Appendix \ref{ap:model_misspecification}.

GWI-net relies on the SVGP kernel defined in \ref{eq:r_SVGP} for its posterior approximation. It therefore inherits numerical instabilities associated with the inversion of the kernel matrix. For the data sets discussed in this paper it was possible to overcome these issues by smart initialisation of the optimiser (cf. Appendix \ref{ap:Implementation}), but it may be an interesting research avenue to come up with a kernel that avoids these instabilities. 

Our method approximates the Wasserstein distance in function space via the spectrum of kernel matrices (cf. Appendix \ref{ap:WD_approximation}). These approximations require quick spectral decay of the composition of prior and variational covariance operator to be accurate and computationally tractable. The prior SE kernel combined with the variational SVGP kernel did have this property (cf. \ref{Ap:wasserstein_estimation error}) which allowed for cheap and accurate approximations. However, other parameterisations may result in less accurate estimation. A theoretical investigation of how the approximation quality relates to kernel properties is an interesting topic for further research.


The proposed framework models prior and variational distribution with a Gaussian measure on the space of square integrable functions. As a consequence the posterior distribution for the functional output is Gaussian as well. This means it is unimodal and concentrated around the posterior mean. Although this constrains the form of functional posterior significantly the authors would argue that the empirical success of GWI-net demonstrates that the approach is flexible to meaningfully quantify uncertainty.

\section{Conclusion}
\vspace{-0.3cm}

In this paper, we developed a framework for generalized variational inference in infinite-dimensional function spaces. We leveraged the function space perspective to develop a new inference approach combining Gaussian measures and Wasserstein distance with predictive performance of deep neural networks, yielding principled uncertainty quantification. The value of our method was demonstrated on several benchmark datasets.

\newpage

\section*{References}

\bibliographystyle{abbrvnat}
\bibliography{biblio}

\newpage

\section*{Checklist}

The checklist follows the references.  Please
read the checklist guidelines carefully for information on how to answer these
questions.  For each question, change the default \answerTODO{} to \answerYes{},
\answerNo{}, or \answerNA{}.  You are strongly encouraged to include a {\bf
justification to your answer}, either by referencing the appropriate section of
your paper or providing a brief inline description.  For example:
\begin{itemize}
  \item Did you include the license to the code and datasets? \answerYes{See Section~\ref{gen_inst}.}
  \item Did you include the license to the code and datasets? \answerNo{The code and the data are proprietary.}
  \item Did you include the license to the code and datasets? \answerNA{}
\end{itemize}
Please do not modify the questions and only use the provided macros for your
answers.  Note that the Checklist section does not count towards the page
limit.  In your paper, please delete this instructions block and only keep the
Checklist section heading above along with the questions/answers below.

\begin{enumerate}

\item For all authors...
\begin{enumerate}
  \item Do the main claims made in the abstract and introduction accurately reflect the paper's contributions and scope?
    \answerYes{} 
  \item Did you describe the limitations of your work?
    \answerYes{} See Appendix \ref{ap:model_misspecification}
  \item Did you discuss any potential negative societal impacts of your work?
    \answerNA{} Nothing to discuss
  \item Have you read the ethics review guidelines and ensured that your paper conforms to them?
    \answerYes{}
\end{enumerate}

\item If you are including theoretical results...
\begin{enumerate}
  \item Did you state the full set of assumptions of all theoretical results?
    \answerYes{}
        \item Did you include complete proofs of all theoretical results?
    \answerYes{} See Appendix A.1-A.6
\end{enumerate}

\item If you ran experiments...
\begin{enumerate}
  \item Did you include the code, data, and instructions needed to reproduce the main experimental results (either in the supplemental material or as a URL)?
    \answerYes{} See footnote in introduction
  \item Did you specify all the training details (e.g., data splits, hyperparameters, how they were chosen)?
    \answerYes{} See Appendix \ref{ap:Implementation} and \ref{ap:Implement_Details_Class} 
        \item Did you report error bars (e.g., with respect to the random seed after running experiments multiple times)?
    \answerYes{}
        \item Did you include the total amount of compute and the type of resources used (e.g., type of GPUs, internal cluster, or cloud provider)?
    \answerYes{} See Appendix \ref{computational resources}
\end{enumerate}

\item If you are using existing assets (e.g., code, data, models) or curating/releasing new assets...
\begin{enumerate}
  \item If your work uses existing assets, did you cite the creators?
    \answerNA{}
  \item Did you mention the license of the assets?
    \answerNA{}
  \item Did you include any new assets either in the supplemental material or as a URL?
    \answerNA{}
  \item Did you discuss whether and how consent was obtained from people whose data you're using/curating?
    \answerNA{}
  \item Did you discuss whether the data you are using/curating contains personally identifiable information or offensive content?
    \answerNA{}
\end{enumerate}

\item If you used crowdsourcing or conducted research with human subjects...
\begin{enumerate}
  \item Did you include the full text of instructions given to participants and screenshots, if applicable?
    \answerNA{}
  \item Did you describe any potential participant risks, with links to Institutional Review Board (IRB) approvals, if applicable?
    \answerNA{}
  \item Did you include the estimated hourly wage paid to participants and the total amount spent on participant compensation?
    \answerNA{}
\end{enumerate}

\end{enumerate}

\newpage

\appendix

\section{Appendix}

\subsection{Bayesian Inference as an Optimization Problem for an Infinite-Dimensional Prior Measure}\label{ap:Bayes_as_Opt}

Let $E$ be a (infinite dimensional) Polish space and $\mathcal{B}(E)$ the Borel $\sigma$-algebra on $E$. We denote the set of Borel probability measures on $\mathcal{B}(E)$ as $\mathcal{P}(E)$ and choose a fixed prior measure $P \in \mathcal{P}(E)$. The likelihood is described by a Markov kernel function $p:  \mathcal{Y} \times E \to [0,\infty)$ with
\begin{equation}
    (y,f) \mapsto p(y|f),
\end{equation}
where $\mathcal{Y} \subset \bbR^N$ is Borel measurable. The prior and the likelihood induce for any fixed $y \in \calY$ a posterior measure denoted as $\widehat{P} \in \mathcal{P}(E)$ \citep[Chapter 1.3]{ghosal2017fundamentals}. 

The next theorem shows that this posterior measure is the solution to a certain optimization problem.

\begin{theorem}[Bayes Posterior as optimization] The Bayesian posterior measure $\widehat{P}$ is given as
\begin{equation}\label{eq:Posterior_As_Optimum}
    \widehat{P} = \underset{Q \in \mathcal{P}(E)}{\text{argmin }} \left\{- \bbE_Q \big[ \log p(y|F) \big] + \bbD_{KL}(Q,P)\right\}
\end{equation}
for any fixed prior measure $P \in \mathcal{P}(E)$ and fixed $y \in \mathcal{Y}$ such that $f \in E \mapsto p(y|f) >0$.
\end{theorem}

\begin{proof}
According to Bayes rule in infinite dimensions \citep[Chapter 1.3]{ghosal2017fundamentals} we know that $\widehat{P}$ is dominated by $P$ with Radon-Nikodym derivative given as
\begin{equation}
    \frac{d \widehat{P}}{d P}(f) = \frac{p(y|f)}{p(y)},
\end{equation}
for $f \in E$ where $p(y):= \int p(y|F=f) \, dP(f) $ is the marginal likelihood for $y$. The reverse is also true and $P$ is dominated by $\widehat{P}$. We prove this by contraposition and therefore assume that $P(A)>0$ for some $A \in \mathcal{B}(E)$. From Bayes rule we know that 
\begin{equation}
    \widehat{P}(A) = \int_A \frac{p(y|f)}{p(y)} \, dP(f) > 0
\end{equation}
as the integrand is positive by assumption and $P(A) >0$. This gives $\widehat{P}(A)>0$ and therefore that $P$ is dominated by $\widehat{P}$. In this case standard rules for Radon-Nikodym derivatives give that 
\begin{equation}
     \frac{d P}{d \widehat{P}}(f) = \frac{p(y)}{p(y|f)},
\end{equation}
for $f \in E$. Note that without loss of generality we can assume that the optimal $Q \in \mathcal{P}(E)$ is dominated by $P$ (and therefore also dominated by $\widehat{P}$) since otherwise \eqref{eq:Posterior_As_Optimum} is infinite by definition of the KL divergence. For such a $Q$ dominated by $P$ it holds that
\begin{align}
   L(Q)&:= - \bbE_Q \big[ \log p(y|F) \big] + \bbD_{KL}(Q,P) \\
       &= - \int \log p(y|f) \, dQ(f)  + \int \log \frac{dQ}{dP}(f) \, dQ(f) \\
       &= - \int \log p(y|f) \, dQ(f)  + \int \log \frac{dQ}{d\widehat{P}}(f) \, dQ(f) + \int \log \frac{d\widehat{P}}{dP}(f) \, dQ(f),
\end{align}
where the last line follows from the chain rule for Radon-Nikodym derivatives. We further have 
\begin{align}
    L(Q) &= - \int p(y|f) \, dQ(f) + \bbD_{KL}(Q, \widehat{P}) + \int \frac{p(y|f)}{p(y)} \, dQ(f) \qquad\text{(Bayes Rule)} \\
    &= \bbD_{KL}(Q, \widehat{P}) + p(y) \\
    &\ge p(y),
\end{align}
since $\bbD_{KL}(Q,P) \ge 0$, with equality if and only if $Q=\widehat{P}$. This proves the claim.
\end{proof}

\subsection{Pointwise Evaluation as Weak Limit}\label{ap:Weak_Limit}
To outline the problem briefly: If $F \sim \mathcal{N}(m,C)$ is a GRE with mean $m \in L^2(\calX,\rho,\bbR)$ and covariance operator $C$ as defined in \eqref{eq:cov_op} then it is in general unclear what the distribution of $F(x)$ would be for a fixed $x \in \calX$. The technical reason is that the pointwise evaluation $\pi_x: L^2(\calX,\rho, \bbR) \to \bbR$, i.e.
\begin{equation}
    \pi_x(f):=f(x) \label{ap:eq:ptw_ev}
\end{equation}
is not well-defined. An element $g$ of the space $L^2(\calX,\rho, \bbR)$ is an equivalence class and only identifiable up to a $\rho$-nullset. This means that the definition of $\pi_x$ in \eqref{ap:eq:ptw_ev} makes no sense whenever $\rho( \{x\}) = 0$ which is the case whenever $\rho$ has a pdf w.r.t. the Lebesgue measure.

However, we will remedy this situation by defining for a fixed $x \in \calX$
\begin{equation}
F(x):= \lim_{n \to \infty} \langle F, h_{n,x} \rangle_2
\end{equation}
where $h_{n,x} \in L^2(\calX,\rho,\bbR)$ is an appropriately chosen sequence and the limit is to be understood as convergence in distribution of the sequence of scalar random variables $ \langle F, h_{n,x} \rangle_2$. 
\begin{theorem}\label{thm:ptws_eval}
Let $F \sim \mathcal{N}(m,C)$ be a GRE in $\mathcal{L}^2(\calX,\rho,\bbR)$ with mean $m \in L^2(\calX,\rho,\bbR)$ and covariance operator $C$ as defined in \eqref{eq:cov_op}. Assume that $\rho$ is a probability measure on $\calX \subset \bbR^D$ and that $\rho$ is absolutely continuous with respect to the Lebesgue measure $\lambda$ on $\bbR^D$ with pdf $\rho'$.
Denote the support of the measure $\rho$ by $supp(\rho)$ and assume that $x$ is an arbitrary point in the interior of $supp(\rho)$ such that $m$, $k$ and $\rho'$ are continuous at $x$.

Let 
\begin{equation}
\eta(t) = \begin{cases}
\exp \Big(-\frac{1}{1-|t|^2} \Big)& \text{if } |t| < 1,\\
0                     & \text{if } |t|\geq 1.
\end{cases}    
\end{equation}
be the so called \textit{standard molifier} and note that $\eta$ is smooth with $\int \eta(t) \, dt = 1$. We further define the sequence  $h_{n,x}(t):= \eta \big(n(t-x)\big)/\rho'(t)$ for $n \in \bbN$, $t \in supp(\rho)$ and $h_{n,x}=0$ for $t \notin supp(\rho)$. Then
\begin{equation}
    \langle F , h_{n,x} \rangle_2 \overset{\mathcal{D}}{\longrightarrow} \mathcal{N}\big(m(x), k(x,x) \big)
\end{equation}
for $n \to \infty$ where $\overset{\mathcal{D}}{\longrightarrow}$ denotes convergence in distribution.

\end{theorem}

\begin{proof}
Note that $supp(h_{n,x}) = B_{1/n}(x):= \{ t \in \bbR^D \, : \, |t-x| \le \frac{1}{n} \}$ and  $B_{1/n}(x) \subset supp(\rho)$ for large enough $n \in \bbN$ since $x$ is from the interior of $supp(\rho)$. This means that $h_{n,x} \in L^2(\calX,\rho,\bbR)$ for large enough $n$ as 
\begin{align}
    \int h_{n,x}(t) \, d \rho(t) &= \int_{supp(\rho)} \left(\frac{\eta \big(n(t-x)\big)}{\rho'(t)} \right)^2 \rho'(t) \, d\lambda(t) \\
    &= \int_{supp(\rho)} \frac{\eta\Big(n(t-x)\Big) }{\rho'(t)} \, dt \\
    &= \int_{B_{1/n}(x)} \frac{\eta\Big(n(t-x)\Big) }{\rho'(t)} \, dt.
\end{align}
The last expression is finite for large enough $n$ because the integrand is continuous at $x$.
According to the definition of of GREs we therefore conclude that 
\begin{equation}
    \langle F, h_{n,x} \rangle_2 \sim \mathcal{N}\big( \langle m, h_{n,x} \rangle_2, \langle C h_{n,x}, h_{n,x} \rangle_2 \big)
\end{equation}
for large enough $n \in \bbN$.

The next statement we show is that $m_n(x):=  \langle m, h_{n,x} \rangle_2 \to m(x) $ for $n \to \infty$. To this end notice that
\begin{align}
    | m_n(x) - m(x) | &=  |\int_{B_{1/n}(x)} h_{n,x}(t) \big(m(x)- m(t) \big)\, d\rho(t)   | \\
    & \le \int_{B_{1/n}(x)} \eta \Big(n(t-x)\Big) |m(x)-m(t)| \, dt. \label{eq:ptws_convergence}
\end{align}
Let now $\epsilon > 0$ be arbitrary. For $n$ large enough we $|m(x)-m(t)|\le \epsilon$ for all $t \in B_{1/n}(x)$ due to the continuity of $m$ in $x$. This immediately implies 
\begin{align}
    \int_{B_{1/n}(x)} \eta \Big(n(t-x)\Big) |m(x)-m(t)| \, dt \le \epsilon \int_{B_{1/n}(x)} \eta \Big(n(t-x)\Big) \, dt = \epsilon,
\end{align}
for large enough $n$ which shows the convergence of $m_n(x)$ to $m(x)$.

A similar argument shows that $k_n(x,x):= \langle C h_{n,x}, h_{n,x} \rangle_2 \to k(x,x)$ for $n \to \infty$. 

We therefore conclude that 
\begin{align}
    \langle F, h_{n,x} \rangle_2 &= \langle F, h_{n,x} \rangle_2 - m_n(x)  + m_n(x) \\
    &= \sqrt{k_n(x,x)} \underbrace{  \frac{\langle F, h_{n,x} \rangle_2 - m_n(x)}{\sqrt{k_n(x,x)}} }_{ \sim \mathcal{N}(0,1)} + m_n(x) \\
    & \overset{\mathcal{D}}{\longrightarrow} \mathcal{N}\big(m(x),k(x,x) \big)
\end{align}
for $n \to \infty$ due to Slutsky's theorem.
\end{proof}

According to Theorem \ref{thm:ptws_eval} we can simply define $F(x) \sim \mathcal{N}(m(x),k(x,x))$ for all $x$ in the interior of the support of $\rho$ if $m$, $k$ and $\rho'$ are continuous at $x$. These are mild assumptions and we can typically assume that they are satisfied in practice.

\subsection{The Wasserstein Metric for Probability Measures}\label{ap:WD}

Let $E$ be a Polish space. For $p\ge1$, let $P_p(E)$ denote the collection of all probability measures $\mu$ on $E$ with finite $p^{\text{th}}$ moment, that is, there exists some $x_0$ in $M$ such that:
\begin{equation}
 \int_M d(x, x_0)^{p} \, \mathrm{d} \mu (x) < \infty.   
\end{equation}

The $p^\text{th}$ Wasserstein distance between two probability measures $\mu$ and $\nu$ in $P_p(E)$ is defined as
\begin{equation}
W_p (\mu, \nu):=\left( \inf_{\gamma \in \Gamma (\mu, \nu)} \int_{E \times E} d(x, y)^p \, \mathrm{d} \gamma (x, y) \right)^{1/p},    
\end{equation}
where $\Gamma(\mu,\nu)$ denotes the collection of all measures on $E \times E$ with marginals $\mu$ and $\nu$ on the first and second arguments respectively. 

More details about the Wasserstein distance can be found in Chapter 7 of \citet{ambrosio2005gradient}.

\subsection{A Tractable Approximation of the Wasserstein Metric}\label{ap:WD_approximation}

Recall that the Wasserstein metric for the two Gaussian measures $P=\mathcal{N}(m_P,C_P)$ and $Q=\mathcal{N}(m_Q,C_Q)$ on the Hilbert space $H= L^2(\calX,\rho,\bbR)$ is given as 
\begin{equation}\label{eq:ap:WD_Gaussians}
    W_2^2(P,Q) = \| m_P - m_Q \|_2^2 + tr(C_P) + tr(C_Q) - 2 \cdot tr \Big[ \big(C_P^{1/2} C_Q^{} C_P^{1/2} \big)^{1/2} \Big].
\end{equation}
Further the operators $C_P$ and $C_Q$ are defined through trace-class kernels $k$ and $r$ as described in Section \ref{subsec:GVI_FSI}. We will now discuss how to approximate each term in \eqref{eq:ap:WD_Gaussians}. 

First, note that 
\begin{equation}
     \| m_P - m_Q \|_2^2 = \int \big(m_P(x) - m_Q(x)\big)^2 \, d \rho(x) \\
     \approx \frac{1}{N} \sum_{n=1}^N  \big(m_P(x_n) - m_Q(x_n)\big)^2,
\end{equation}
which follows by replacing the true input distribution with the empirical data distribution.
Second, note that under very general conditions on $k$ and $\rho$ it holds that \citep{brislawn1991traceable} 
\begin{equation}
    tr(C_P) =  \int k(x,x) \, d\rho(x)
\end{equation}
and similarly for $C_Q$. Again by replacing $\rho$ with the empirical data distribution we obtain natural estimators:
\begin{align}
     &tr(C_P) \approx \frac{1}{N} \sum_{n=1}^N k(x_n,x_n), \\
     &tr(C_Q) \approx \frac{1}{N} \sum_{n=1}^N r(x_n,x_n).
\end{align}
Denote by $\lambda_n(C)$ the $n$-th eigenvalue of a positive, self-adjoint operator $C$. By definition of the trace and the square root of an operator we have 
\begin{align}
    tr \Big[ \big(C_P^{1/2} C_Q^{} C_P^{1/2} \big)^{1/2} \Big] 
    &= \sum_{n=1}^\infty \sqrt{ \lambda_n \Big( C_P^{1/2} C_Q^{} C_P^{1/2}  \Big) } \\
    &= \sum_{n=1}^\infty \sqrt{ \lambda_n \Big( C_Q C_P  \Big)},\label{eq:spec_CPQ}
\end{align}
where the second line follows from the fact that the operator $C_Q C_P$ has the same eigenvalues as $C_P^{1/2} C_Q^{} C_P^{1/2}$ \citep[Proposition 1]{hladnik1988spectrum}. The operator $C_Q C_P$ is given as
\begin{align}
    C_Q C_P g(x) &= \int r(x,x') (C_Pf)(x') \, d \rho(x') \\
                 &= \int r(x,x') \big( \int k(x',t) f(t) d \rho(t) \big) \, d \rho(x') \\
                 &= \int \int r(x,x') k(x',t) f(t) \, d\rho(x') d \rho(t) \\
                 &= \int ( r*k)(x,t) f(t) \, d \rho(t), 
\end{align}
where we define 
\begin{equation}
    (r*k)(x,t) :=   \int r(x,x') k(x',t) \, d\rho(x') 
\end{equation}
for all $x,t \in \calX$. This means that $C_{Q} C_P$ is also an integral operator with (non-symmetric) kernel $r*k$. We again replace $\rho$ with $\widehat{\rho}$ to obtain
\begin{equation}
    \widehat{ (r*k)}(x,t) = \frac{1}{N} \sum_{n=1}^N  r(x,x_n) k(x_n, t).
\end{equation}
The spectrum of $C_Q C_P$ can now be approximated by the spectrum of the  matrix $ \frac{1}{N} \widehat{(r*k)}(X,X)$ \citep[cf. Chapter 4.3.2]{rasmussen2003gaussian} or $ \frac{1}{N_S} \widehat{  (r*k)}(X_S,X_S)$ where $X_S$ is a subsample of the data points $X$ of size $N_S <N$. If we plug this approximation into \eqref{eq:spec_CPQ} we obtain 
\begin{align}
    tr \Big[ \big(C_P^{1/2} C_Q^{} C_P^{1/2} \big)^{1/2} \Big] &\approx \sum_{m=1}^{N_S} \sqrt{\lambda_m\big( \frac{1}{N_S} \widehat{  (r*k)}(X_S,X_S) \big)} \\
    &= \frac{1}{\sqrt{N_S}} \sum_{m=1}^{N_S} \sqrt{ \lambda_m \Big( \frac{1}{N} r(X_S,X)k(X,X_S) \Big)},
\end{align}
which is the last expression that we had to approximate.

Note that since $C_Q C_P$ has the same spectrum as the self-adjoint, positive, trace-class operator $C_P^{1/2} C_Q^{} C_P^{1/2}$ we know that its eigenvalues are real, positive and converge to zero.

\subsection{Generalized Loss for Regression in Batch Mode}\label{ap:Batch_Mode}

The batch version of the generalized loss is given as:
\begin{align}
    \widehat {\mathcal{L}} &= \frac{N}{2} \log( 2 \pi \sigma^2) + \frac{N}{{N_B}} \sum_{b=1}^{{N_B}} \frac{\big( y_{n_b} - m_Q(x_{n_b}) \big)^2 +r(x_{n_b},x_{n_b})}{2 \sigma^2} + \frac{1}{{N_B}} \sum_{b=1}^{{N_B}} \big( m_P(x_{n_b})-m_Q(x_{n_b}) \big)^2 \\
&+ \frac{1}{{N_B}} \sum_{b=1}^{{N_B}} k(x_{n_b},x_{n_b}) + \frac{1}{{N_B}} \sum_{b=1}^{{N_B}} r(x_{n_b},x_{n_b}) - \frac{2}{\sqrt{N_B N_S}} \sum_{s=1}^{N_S} \sqrt{\lambda_s\big(  r(X_S,X_B)k(X_B,X_S) \big)},
\end{align}
$N_B \in \bbN$ is the batch-size. The indices $n_1,\hdots,n_{N_B}$ are the batch-indices and $X_B$ is the batch matrix.

\subsection{GWI for (Multiclass) Classification}\label{ap:Classification}
Let $\{(x_n,y_n)\}_{n=1}^N \subset \calX \times \calY$ be data with $\calX \subset \bbR^D$ and $\calY=\{1,\hdots,J\}$, where $J \in \bbN$ represents $J\ge2$ distinct classes. 

\shortparagraph{Model} We use the same likelihood for $y:=(y_1,\hdots,y_N)$ as described in Chapter 4 of \citet{matthews2017scalable} which is:
\begin{equation}
    p(y|f_1,\hdots,f_J) = \prod_{n=1}^N p(y_n|f_1,\hdots,f_J)
\end{equation}
with 
\begin{equation}
    p(y_n|f_1,\hdots,f_J) := h_{y_n}^{\epsilon}\big(f_1(x_n),\hdots,f_J(x_n)\big),
\end{equation}
for $ y_n \in \{1,\hdots,J\}$. The function $h_\ell^{\epsilon}$ is defined as 
\begin{equation}
    h_\ell^{\epsilon}(t_1,\hdots,t_J) \begin{cases}
1- \epsilon ~&\text{ if } \ell = \underset{j=1,\hdots,J}{\text{argmax}}\{t_j\}, \\
\frac{\epsilon}{J-1} ~&\text{ if } \text{otherwise}.
\end{cases}
\end{equation}
for $\ell=1,\hdots,J$ for $\epsilon >0$. We chose $\epsilon=1\%$ in our implementation.

We assume that $F_1,\hdots F_J$ are independent GREs on ${L}^2(\calX,\rho,\bbR)$ with prior means $m_{P,j}$ and prior covariance operators $C_{P,j}$, $j=1,\hdots,J$.

The variational measures for $F_1,\hdots,F_J$ are assumed to be independent and given as $Q_j = \mathcal{N}\big(m_{Q,j}, C_{Q,j} \big)$ for $j=1,\hdots,J$. We further write $\bbQ\Big(  \big(F_1(x), \hdots, F_J(x) \big) \in A \Big)$, $A \subset \bbR^J$ for the variational (posterior) approximation of the probability of the event $ \{ \big(F_1(x), \hdots, F_J(x) \big) \in A \}$.

This leads to the following expected log-likelihood
\begin{align}
    &\bbE_{\bbQ} \big[ \log p(y|F_1,\hdots,F_J) \big]  \\
    &= \sum_{n=1}^N \bbE_{\bbQ} \big[ \log p(y_n|F_1,\hdots,F_J) \big] \\
    &= \sum_{n=1}^N \log(1-\epsilon) \bbQ\big( \underset{j=1,\hdots,J}{\text{argmax}}\{F_j(x_n)\} = y_n \big) + \log(\frac{\epsilon}{J-1})  \bbQ\big( \underset{j=1,\hdots,J}{\text{argmax}}\{F_j(x_n)\} \neq y_n \big) \\
    &\approx\sum_{n=1}^N \log(1-\epsilon) S(x_n,y_n) + \log( \frac{\epsilon}{J-1}) \big(1-S(x_n,y_n) \big),\label{eq:multiclass_ell}
\end{align}
with 
\begin{equation}
S(x,j):= \frac{1}{\sqrt{\pi}} \sum_{i=1}^I w_i \prod_{l \neq j} \phi \Big( \frac{\sqrt{2 r_j(x,x)} \xi_i + m_{Q,j}(x) - m_{Q,l}(x) }{\sqrt{r_l(x,x)}} \Big)
\end{equation}
for any $x \in \calX$, $j=1,\hdots,J$ where $(w_i,\xi_i)_{i=1}^I$ are the weights and roots of the Hermite polynomial of order $I \in \bbN$. This is the same Gauss-Hermite approximation as described in Chapter 4 of \citet{matthews2017scalable}.

The final objective for multiclass classification is given as
\begin{equation}
    \mathcal L = - \bbE_Q \big[ \log p(y|F_1,\hdots,F_J) \big] + \sum_{j=1}^J W_2^2(P_j, Q_j),
\end{equation}
where the expected log-likelihood is approximated by \eqref{eq:multiclass_ell} and each Wasserstein distance $W_2^2(P_j,Q_j)$ can be estimated as in \eqref{eq:WD_approx1}-\eqref{eq:WD_approx2}.

\shortparagraph{Prediction} The probability that an unseen point $x^* \in \calX$ belongs to class $j \in \{1,\hdots,J\}$ is given as
\begin{equation}
    \bbQ( Y^* = j) = (1-\epsilon) S(x^*,j) + \frac{\epsilon}{J-1} \big( 1- S(x^*,j) \big)
\end{equation}
for any $x^* \in \calX$. We predict the class label as maximiser of this probability. If we apply tempering, we simply replace every $r_j(x,x)$ with $T \cdot r_j(x,x)$ for $j=1,\hdots,J$ in the definition of $S(x,j)$.

\shortparagraph{Negative Log Likelihood} The variational approximation to the negative log-likelihood is 
\begin{equation}
    NLL = - \log \Big[ (1-\epsilon) S(x^*,y^*) + \frac{\epsilon}{J-1} \big( 1- S(x^*,y^*) \big) \Big]
\end{equation}
for any point $x^* \in \calX$ for which we know that the class label is $y^* \in \{1,\hdots,J\}$.

\subsection{Implementation Details: Regression}\label{ap:Implementation}

The Regression model is given as $F \sim \mathcal{N}(0,C)$ and 
\begin{equation}
    Y_n = F(x_n) + \epsilon_n
\end{equation}
with $\epsilon_n \sim \mathcal{N}(0, \sigma^2)$, $n=1,\hdots,N$. The covariance operator $C_P$ depends on the choice of a kernel $k$, i.e. $C_P=C_{P,k}$ for which we use the ARD kernel $k$ given as
\begin{equation}
 k(x,x') = \sigma_f^2 \exp \Big( - \frac{1}{2} \sum_{d=1}^D  \frac{(x_d - x'_d)^2}{\alpha_d^2}  \Big)  \end{equation}
for $x,x' \in \bbR^D$. We refer to $\sigma_f > 0$ as \textit{kernel scaling factor}, to $\alpha_d >0$ as \textit{length-scale} for dimension $d$ and to $\sigma>0$ as \textit{observation noise}.

The data is first randomly split into three categories: training set $80\%$, validation set $10\%$ and test set $10 \%$. The observations $Y$ are then standardised by subtracting the empirical mean (of the training data) and dividing by the empirical standard deviation (of the training data). The inputs data $X$ is left unaltered.

\shortparagraph{The number of inducing points} The number of inducing points $M$ is treated as a hyperparameter, this means we train the model for each $M \in \{0.5 \sqrt{N},\sqrt{N}, 1.5 \sqrt{N}, 2\sqrt{N} \}$ and choose the best model. For GWI: SVGP we use $M \in \{1 \sqrt{N}, 2\sqrt{N}, \hdots 5 \sqrt{N} \}$.

\shortparagraph{The choice of inducing points} The input points $Z_1,\hdots,Z_M$ in \eqref{eq:r_SVGP} are sampled independently from the training data $X$ and then fixed for GWI-net. For GWI: SVGP they are only initialised this way and then learned by maximising the generalized loss.

\shortparagraph{Prior hyperparameters} The prior hyperparameters $\sigma_f$, $\alpha:=(\alpha_1,\hdots,\alpha_D) $ and $\sigma$ are chosen by maximising the marginal log-likelihood for the data $X=Z$ and the corresponding observations, which we denote $Y_Z$. Note that the marginal log-likelihood is tractable and given as 
    \begin{equation}
        \log p(y_Z) = -\frac{1}{2} \log \Big( \det \big( k(Z,Z) + \sigma^2 I_M \big) \Big) - \frac{1}{2} {y_Z}^T \big( k(Z,Z) + \sigma^2 I_M \big)^{-1} {y_Z}.
    \end{equation}
and can therefore be evaluated in $\mathcal{O}(M^3)=\mathcal{O}(N \sqrt{N})$.
\shortparagraph{Variational mean} For GWI-net we use a neural network with $L=2$ hidden layers, width $D_1=D_2=10$ and tanh as activation function. This follows the set-up of \citet{ma2021functional}.

\shortparagraph{Variational kernel} The kernel $r$ which is chosen as described in \eqref{eq:r_SVGP} and therefore depends on the covariance matrix $\Sigma \in \bbR^{M \times M}$ and the $M \in \bbN$ inducing points $Z=(Z_1,\hdots,Z_M) \in \bbR^{D \times M}$. We parametrise $\Sigma$ as $\Sigma = L L^T$ with initialisation
    \begin{equation}
        L = \text{Chol} \Big(  \big(k(Z,Z) + \frac{1}{\sigma^2} k(Z,X) k(X,Z) \big)^{-1}  \Big),
    \end{equation}
    where $k(Z,X) k(X,Z)$ is approximated by batch-sizing as $ \frac{N}{N_B} k(Z,X_B) k(X_B,Z)$. This corresponds to an approximation of the optimal choice for $\Sigma$ in SVGP \citep{titsias2009variational}.
    
\shortparagraph{Parameters in the generalized loss} The generalized loss in Appendix \ref{ap:Batch_Mode} depends further on $N_S$, $N_B$ and $X_S$. The batch-size $N_B$ is chosen to be $N_B=1000$ for $N>1000$. For $N<1000$ we use the full training data. The comparison points $X_S$ are sampled independently from the training data $X$ in each iteration. We train here for 1000 epochs on the regression task and 100 epochs on the classification task following \cite{ma2021functional}.

\shortparagraph{Tempering the predictive posterior}

\citet{wenzel2020good} observe that the performance of many Bayesian neural networks can be improved by \textit{tempering} the predictive posterior. Tempering refers to a shrinking of the predictive posterior variance by a factor of $\alpha_T \in [0,1]$. This effect has also been observed for Gaussian processes in \citet{adlam2020cold} where it can be interpreted as elevating problems that occur from prior misspecification. The prior hyperparameters for the ARD kernel $k$ in \eqref{eq:ARD_kernel} are selected by maximising the marginal log-likelihood on a subset of the training data. This procedure may lead to prior misspecification, which is why we decided to temper the predictive posterior, which means that we use the predictive distribution 
\begin{equation}\label{eq:predictive_distr}
    Y^*|Y \sim \mathcal{N}\Big( m_Q(x^*), \alpha_T \big( r(x^*,x^*) + \sigma^2 \big) \Big)
\end{equation}
for an unseen data point $x^* \in \calX$. The (tempered) NLL for each data point is given as
\begin{align}
    \text{NLL} &:= - \log p_{\alpha_T}(y^*|y) \\
    &= \frac{1}{2} \log \Big(  \alpha_T \cdot (r(x^*,x^*) + \sigma^2) \Big) + \frac{1}{2} \frac{(y-y^*)^2}{ \alpha_T \cdot (r(x^*,x^*) + \sigma^2)} + \frac{1}{2} \log( 2 \pi). \label{ap_eq: NLL}
\end{align}
The tempering factor $\alpha_T$ is chosen as minimiser of the average NLL on the validation set. The final predictions on the test set are made using this optimal $\alpha_T$ and \eqref{eq:predictive_distr}. Note however that for the NLL numbers reported in Table \ref{tab:UCI} we add $\log(\widehat{\sigma}_{train})$ to \eqref{ap_eq: NLL} where $\widehat{\sigma}_{train}$ is the empirical standard deviation of the training data. This is done for fair comparison as it is how the NLL is calculated in \cite{ma2021functional}.

\subsection{Implementation Details: Classification}\label{ap:Implement_Details_Class}

As described in section \eqref{ap:Classification} we use the prior mean functions $m_{P,j}$ and kernels $k_{j}$ for $j=1,\hdots,J$. For our experiments we chose $m_{P,j}=0$ for $j=1,\hdots,J$ and $k:=k_{1}=\hdots,k_J$ where $k$ is the ARD kernel in \eqref{eq:ARD_kernel}.

We use a multi-output neural network for the variational means $m_{Q,j}$ and an SVGP kernel for each $r_{j}$ , $j=1,\hdots,J$.

\shortparagraph{The number of inducing points} The number of inducing points $M$ is treated as a hyperparameter, this means we train the model for each $M \in \{0.5 \sqrt{N},0.75\sqrt{N},\sqrt{N}\}$ and choose the best model.

\shortparagraph{The choice of inducing points} The input points $Z_1,\hdots,Z_M$ in \eqref{eq:r_SVGP} are sampled independently from the training data $X$ and then fixed for GWI-net. 

\shortparagraph{Prior hyperparameters} The prior hyperparameters are initialised as described in \ref{ap:Implementation}, thus maximising the marginal likelihood of a \emph{regression} model, since the marginal likelihood of our classification model is intractable.

\shortparagraph{Variational mean} We use the same CNN architecture as described in \citet{immer2021improving,schneider2019deepobs} for all models.

\shortparagraph{Variational kernel} Each variational kernel $r_j$ uses the same inducing points $Z$ but gets an individual matrix $\Sigma^{j} \in \bbR^{M\times M}$ for $j=1,\hdots,J$. They are all initialised as described in \ref{ap:Implementation}.

\shortparagraph{Parameters in the generalized loss} The generalized loss in Appendix \ref{ap:Batch_Mode} depends on $N_S$, $N_B$ and $X_S$. The batch-size $N_B$ is chosen to be $N_B=1000$ for $N>1000$. For $N<1000$ we use the full training data. The comparison points $X_S$ are sampled independently from the training data $X$ in each iteration. We train 100 epochs on the classification task following \cite{ma2021functional}.

\shortparagraph{Tempering the predictive posterior} For the same reasons as outlined in Appendix \ref{ap:Implementation} we temper the predictive posterior. Recall that the NLL for classification is given as
\begin{equation}
    NLL = - \log \Big[ (1-\epsilon) S(x^*,y^*) + \frac{\epsilon}{J-1} \big( 1- S(x^*,y^*) \big) \Big]
\end{equation}
for any point $x^* \in \calX$ for which we know that the class label is $y^* \in \{1,\hdots,J\}$. We use a tempering factor $\alpha_j>0$ for each variational measure $Q_j \sim \mathcal{N}(m_{Q,j}, \alpha_j r_j)$, $j=1,\hdots,J$. We train the model with $\alpha_j=1$ for all $j=1,\hdots,J$ and select the tempering factors afterwards as minimiser of the average NLL on the validation set.

\subsection{Illustrative Example for Two Dimensional Inputs}\label{ap:Illustrative_Example}

In \citet{foong2020expressiveness} it is observed that several BNN posterior approximation techniques struggle with the quantification of in-between uncertainty. The red points mark where observations were made and it is clear that mean-field variational inference (MVFI) \citep{hinton1993keeping} and Monte Carlo Dropout (MCDO) \citep{gal2016dropout} exhibit unjustifiably high posterior certainty in the area where no observations are made. This is a pathology of the approximation technique as the true Bayesian posterior which is approximated to very high precision by Hamiltonian Monte Carlo (HMC) \citep{neal2012bayesian} or the infinite-width GP limit \citep{matthews2018gaussian} do not display such behaviour.

In \Cref{fig:2D_dataset} our method GWI-net is displayed next to the methods described in \citet{foong2020expressiveness}. As one can observe our model is keenly aware of its limited ability to predict points  in-between the two clusters of observed data points.

\begin{figure}[htb!]
    \centering
        \begin{subfigure}[b]{0.2\columnwidth}
        \includegraphics[width=\textwidth]{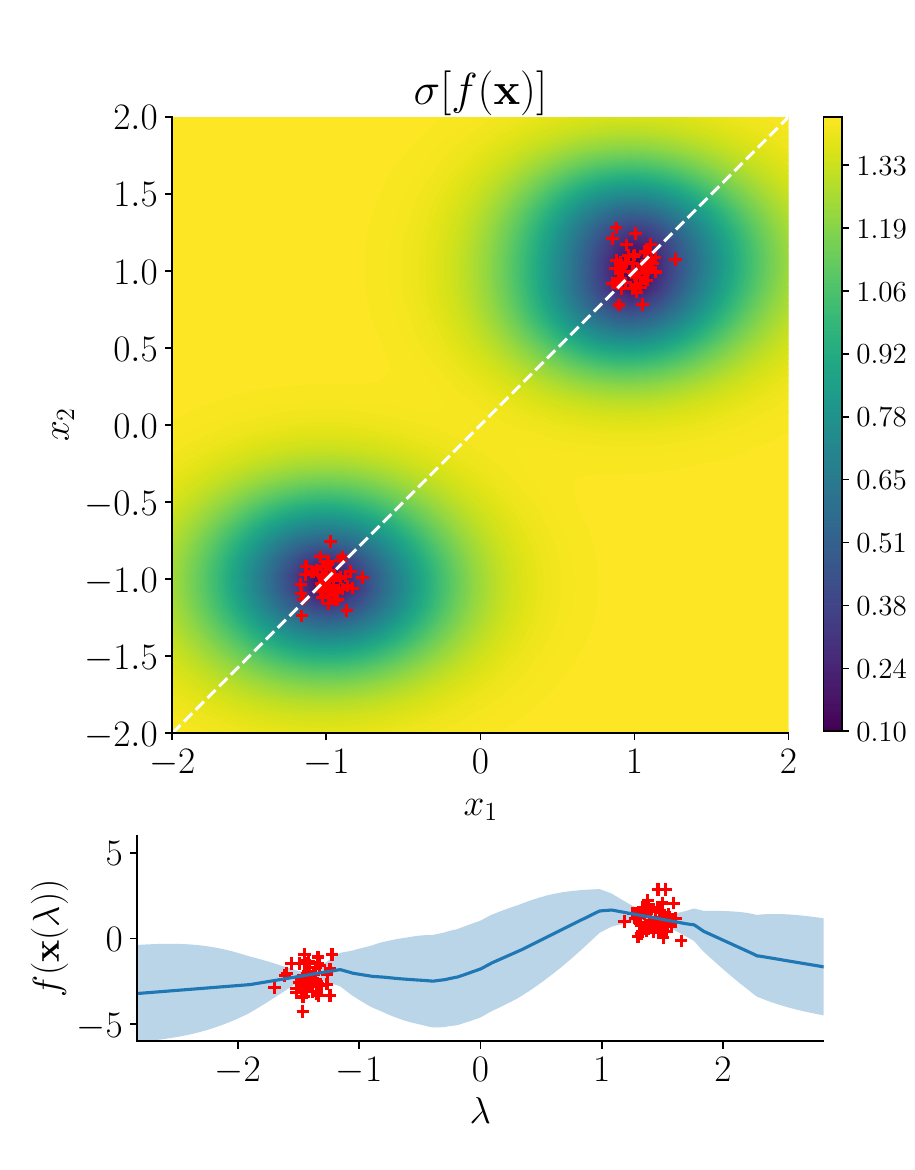}
        \caption{GWI}
    \end{subfigure}%
    \begin{subfigure}[b]{0.2\columnwidth}
        \includegraphics[width=\textwidth]{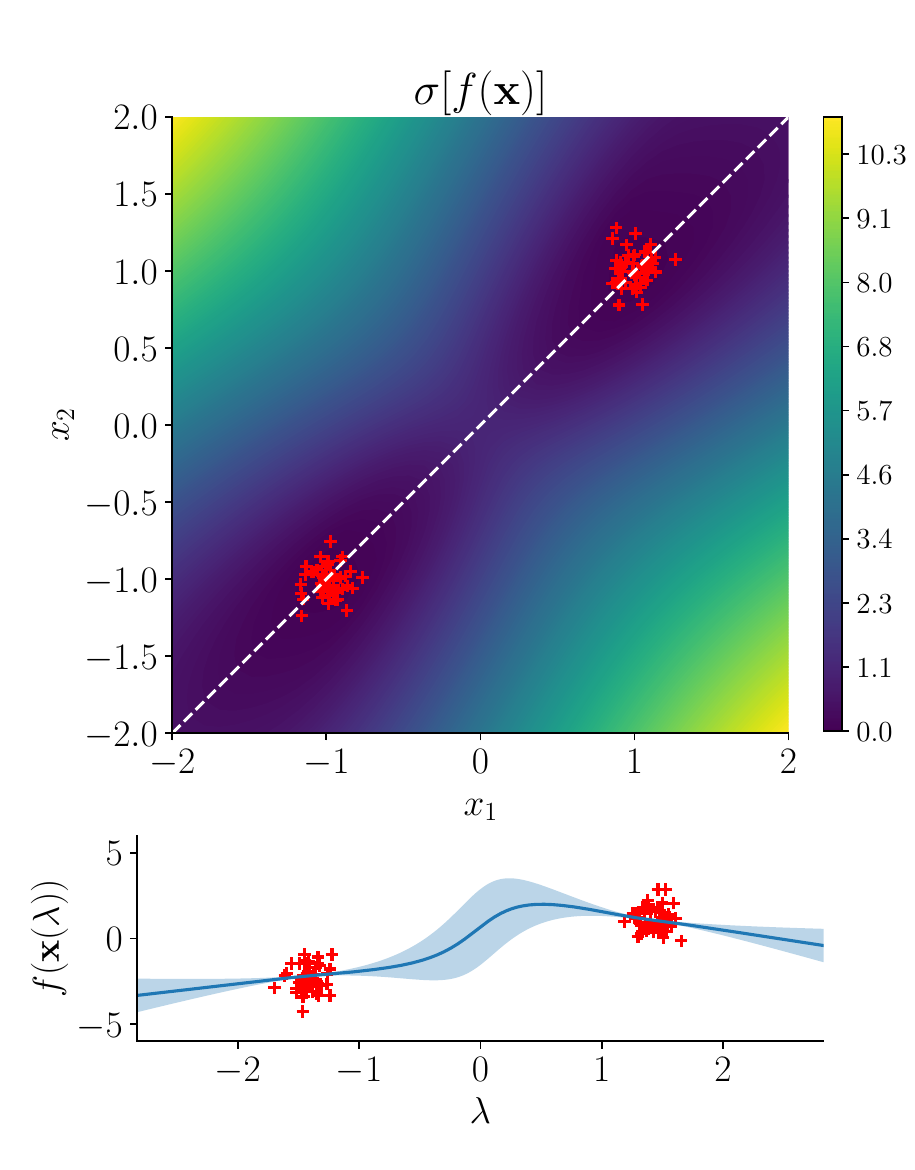}
        \caption{Inf-width limit GP}
    \end{subfigure}%
    \begin{subfigure}[b]{0.2\columnwidth}
        \includegraphics[width=\textwidth]{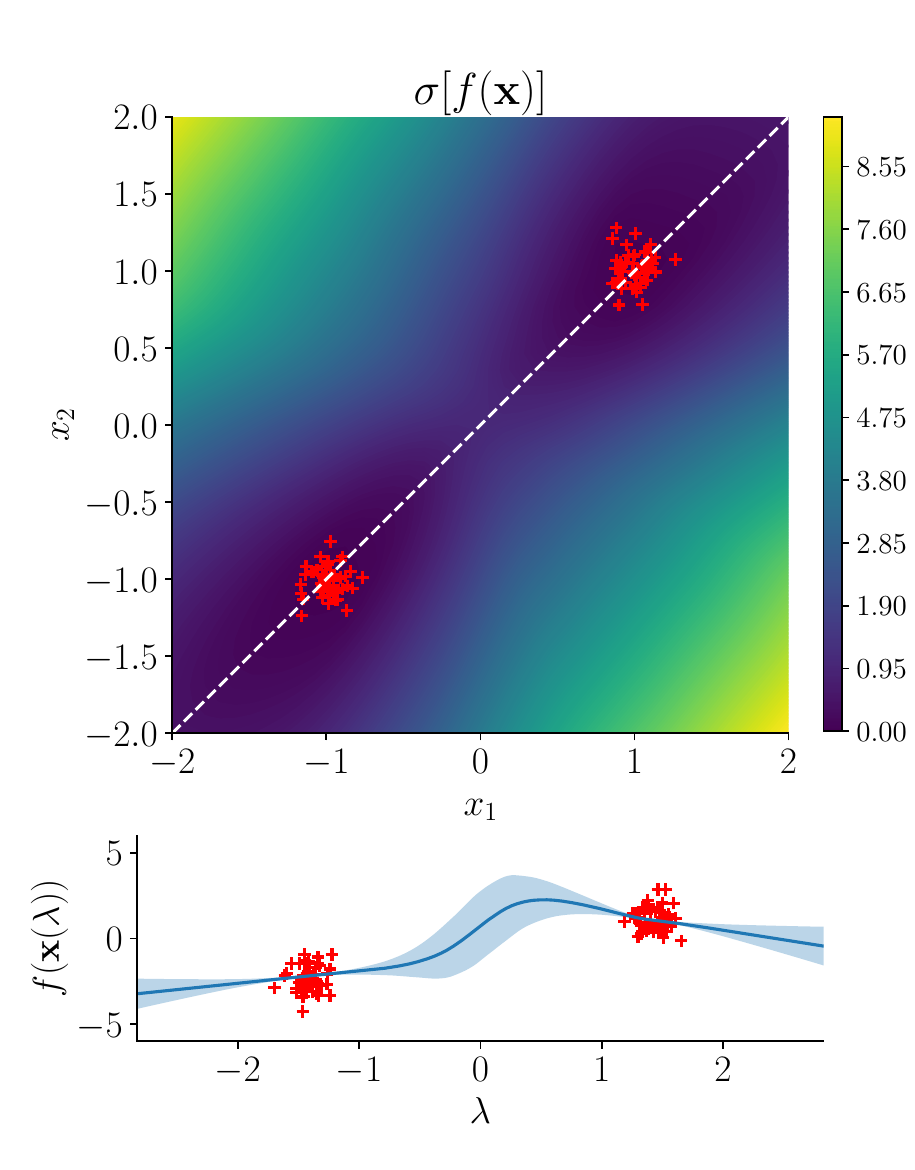}
        \caption{HMC}
    \end{subfigure}%
    \begin{subfigure}[b]{0.2\columnwidth}
        \includegraphics[width=\textwidth]{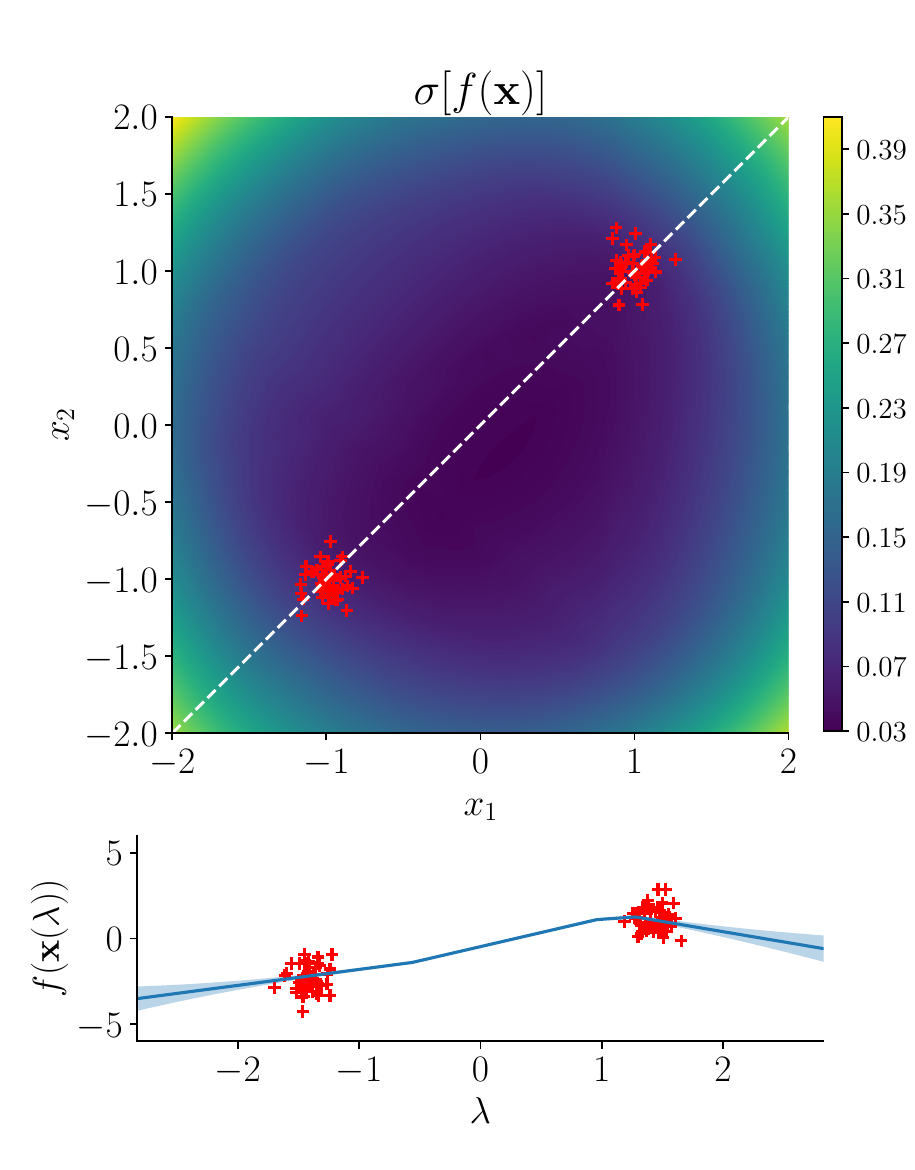}
        \caption{MFVI}
    \end{subfigure}%
    \begin{subfigure}[b]{0.2\columnwidth}
        \includegraphics[width=\textwidth]{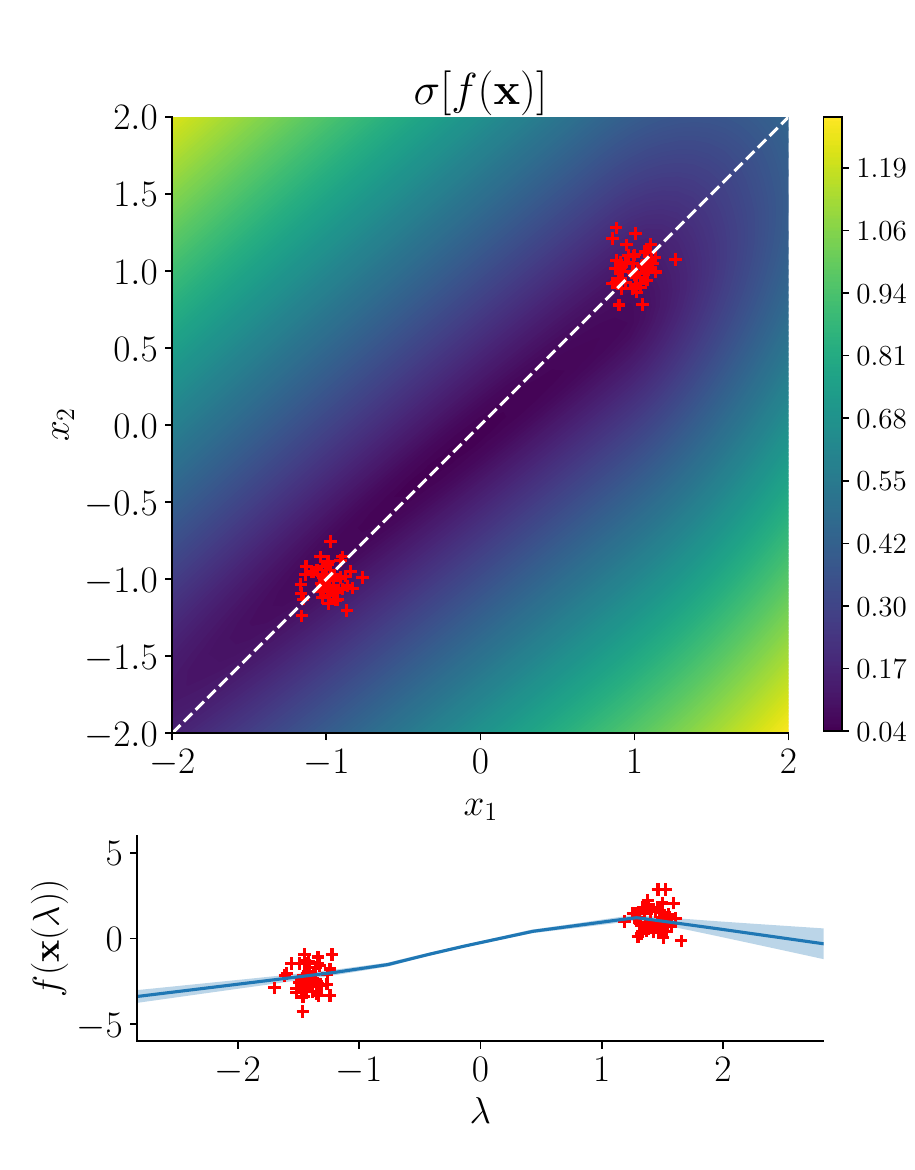}
        \caption{MCDO}
    \end{subfigure}
    \caption{Regression on a 2D synthetic dataset (red crosses). The colour plots show the standard deviation of the output, $\sigma[f(\mathbf{x})]$, in 2D input space. The plots beneath show the mean with 2-standard deviation bars along the dashed white line (parameterised by $\lambda$). MFVI and MCDO are overconfident for $\lambda \in [-1, 1]$.}
    \label{fig:2D_dataset}
    \vspace{-.3cm}
\end{figure}

\subsection{Model Misspecification in Gaussian Wasserstein Inference}\label{ap:model_misspecification}

The generalized loss in Appendix \ref{ap:Batch_Mode} is a valid optimization target for any $m_P,m_Q \in L^2(\calX,\rho,\bbR)$ and any trace-class kernels $k$ and $r$. This gives the user a lot of abilities to specify different models, by experimenting with various choices, specifically for $m_Q$ and $r$. However with great power comes great responsibility: it is quite easy to misspecify GWI. To illustrate the issue let us use a periodic kernel $k$ \citep{duvenaud2014kernel} given as 
\begin{equation}
k(x,x'):= \sigma_f^2 \exp\big( - \frac{1}{\alpha^2} \sin^2 ( \pi |x-x'|/p ) \big)
\end{equation}
and the SVGP kernel $r$ in \eqref{eq:r_SVGP}. By the definition of $r$ the uncertainty will be low for points \textit{similiar} to the inducing points $Z$, i.e. for points $x \in \calX$ $k(x,z_m) \approx \sigma_f^2$ for all $m=1,\hdots,M$. A problem now occurs, if the posterior mean $m_Q$ does not respect the knowledge embedded in $k$ and $r$. Lets for example use a simple fully connected deep neural network $m_Q$ and choose the point $x^*:=z_1+ 10p $. Assume further that $z_1,\hdots,z_M < x^*$. Then we get $k(x^*,z_m) = k(z_1,z_m)$ for all $m=1,\hdots,M$ due to the periodicity of $\sin(x)$ and therefore $r(x^*,x^*)=r(z_1,z_1)$. It is however very unlikely that the neural network will predict $m_Q(z_1)$ as well as $m_Q(x^*)$ since it is unaware of this periodicity.

This small example should illustrate that it is crucial that $m_Q$ is compatible with the prior knowledge reflected in $k$ and $r$. However, note that this problem is not present for our model, GWI-net. The ARD kernel encodes the inductive bias that the underlying function is infinitely differentiable and that points close to each other have highly correlated functional outputs. A simple fully connected DNN with tanh activation function is indeed smooth and further it is reasonable to assume that predictions are more unreliable the further they are from the data (as measured by the squared euclidean distance). The ARD kernel is in this sense compatible with a fully connected DNN.

It shall be noted that the DNN used for the classification examples in \eqref{sec:exp_classification} used convolutional layers as explained in Appendix \ref{ap:Implement_Details_Class}. This can be understood as embedding prior knowledge about translation equivariance into the DNN \citep[Chapter 9.4]{goodfellow2016deep}. It might therefore be desirable to use a prior kernel $k$ that embeds similar properties such as the kernel suggested by \citet{van2017convolutional}. We considered this to be beyond the scope of this paper but the interaction of DNN architecture and the choice of prior kernels is an interesting avenue for future research.

\subsection{Details on computational resources used}
\label{computational resources}
For all our experiments, we distributed our jobs across 8 Nvidia V100 cards.

\subsection{Additional plots for 1D experiments}\label{ap:addtional_plots}
In \Cref{1d_illustration_2} we compare GWI-net, GWI-SVGP and SVGP on one-dimensional toy data. Note that all three methods use the same posterior kernel, but GWI-net differs from GWI-SVGP in terms of the posterior mean function. GWI-SVGP and SVGP have the same posterior mean but differ in terms of the objective function used for training.
\begin{figure}[htb!]
    \centering
    \includegraphics[width=0.33\linewidth]{1d_plots/GWI_1_1000.png}%
        \includegraphics[width=0.33\linewidth]{1d_plots/GWI_2_1000.png}%
            \includegraphics[width=0.33\linewidth]{1d_plots/GWI_3_1000.png}

    \includegraphics[width=0.33\linewidth]{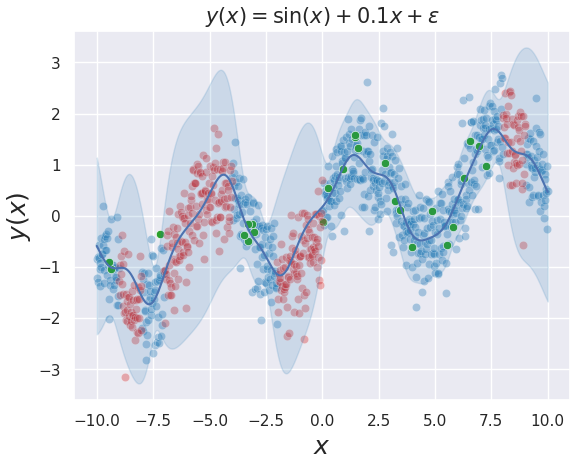}%
        \includegraphics[width=0.33\linewidth]{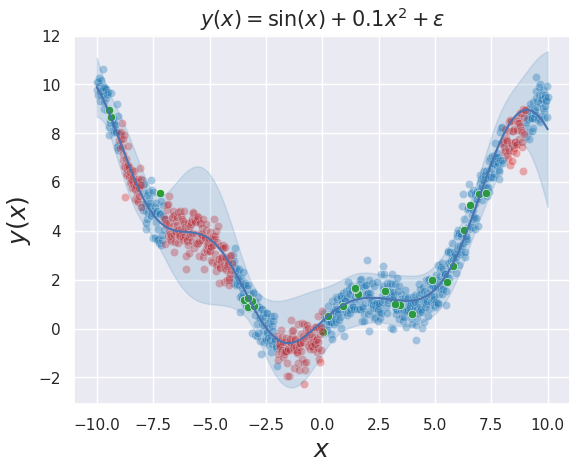}%
            \includegraphics[width=0.33\linewidth]{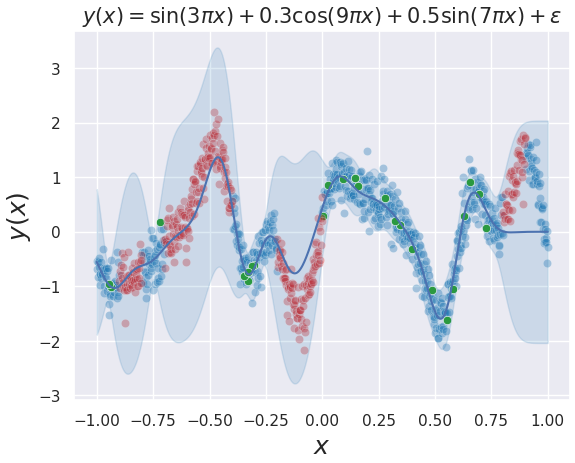}

    \includegraphics[width=0.33\linewidth]{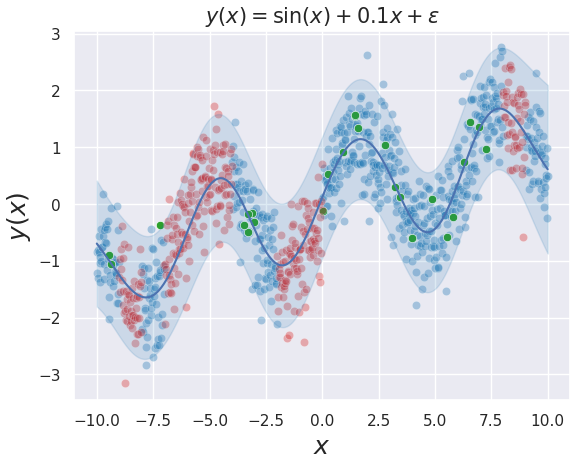}%
        \includegraphics[width=0.33\linewidth]{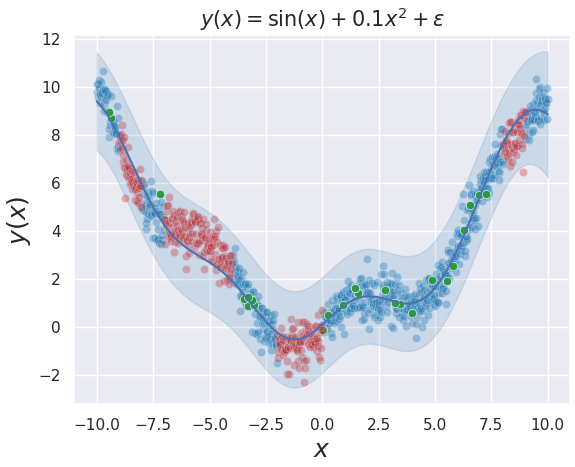}%
            \includegraphics[width=0.33\linewidth]{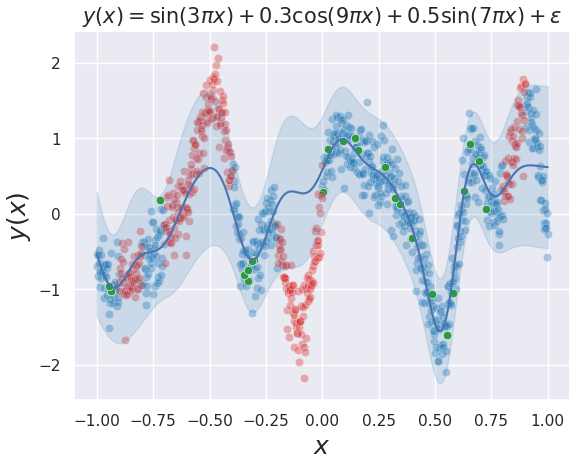}
            
    \caption{{\color{colorA} \samplesquare{}}: Training data $\quad$ {\color{colorB} \samplesquare{}}: Unseen data $\quad$ {\color{colorC} \samplesquare{}}: Inducing points  \\
    We query the above functions at $N=1000$ equidistant points and add white noise with $\epsilon \sim \mathcal{N}(0,0.5^2)$. We use $M=30$ inducing points and train our method as described in Appendix \ref{ap:Implementation}. The plot shows $m_Q(x)\pm1.96\sqrt{\mathbb{V}[Y^*(x)|Y]}$ where $\mathbb{V}[Y^*(x)|Y]$ is the posterior predictive variance given as $r(x,x)+\sigma^2$. Here the fitted models from top to bottom are GWI-net, GWI-SVGP and SVGP. }
    \label{1d_illustration_2}
\end{figure}
\subsection{Empirical estimation error of 2-Wasserstein distance}\label{Ap:wasserstein_estimation error}

The approximation quality of the 2-Wasserstein distance is determined by the approximation quality of the spectrum of the appearing covariance operators. For most kernels in practice like SE or Matern kernel, the spectrum decays very quickly, which is why using the first 100 eigenvalues often empirically seems to be sufficient to approximate the spectrum and therefore the 2-Wasserstein distance. We plot the magnitude of the first 100 positive eigenvalues (sorted on magnitude) for datasets BOSTON, CONCRETE, ENERGY, WINE and YACHT in \Cref{fig:eigenvalues}.

\begin{figure}[htb!]
    \centering
        \begin{subfigure}[b]{\columnwidth}
        \includegraphics[width=\textwidth]{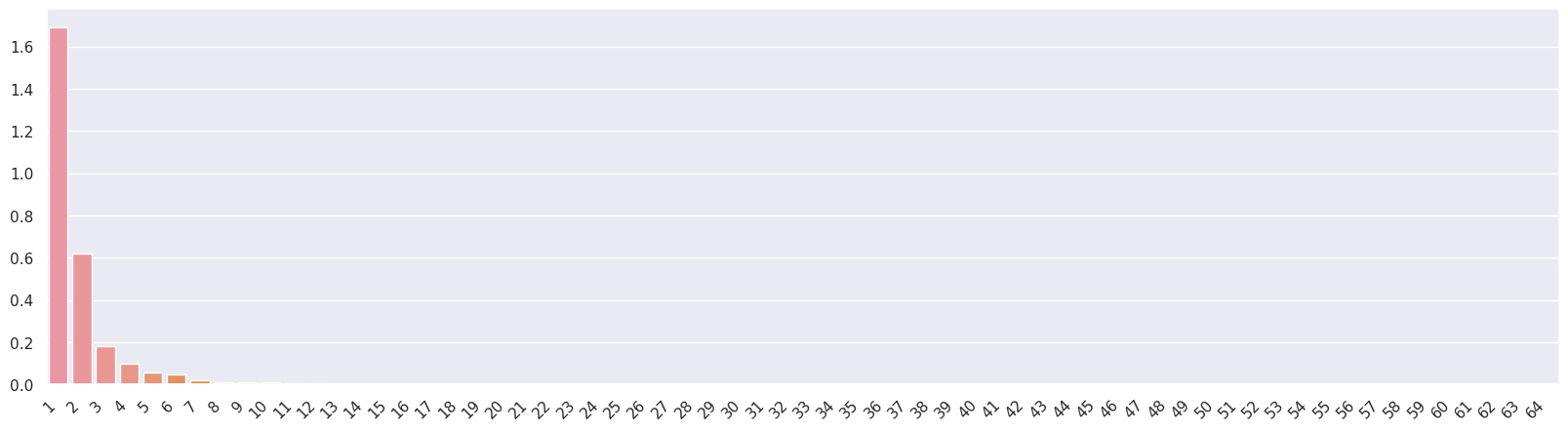}
        \caption{BOSTON}
    \end{subfigure}
    \begin{subfigure}[b]{\columnwidth}
        \includegraphics[width=\textwidth]{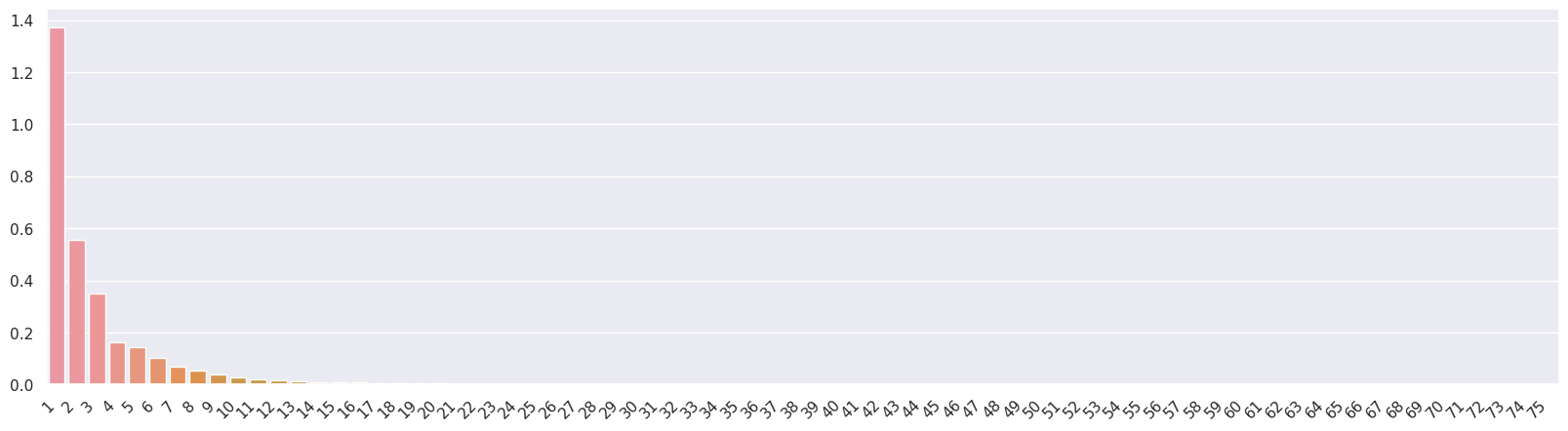}
        \caption{CONCRETE}
    \end{subfigure}
    \begin{subfigure}[b]{\columnwidth}
        \includegraphics[width=\textwidth]{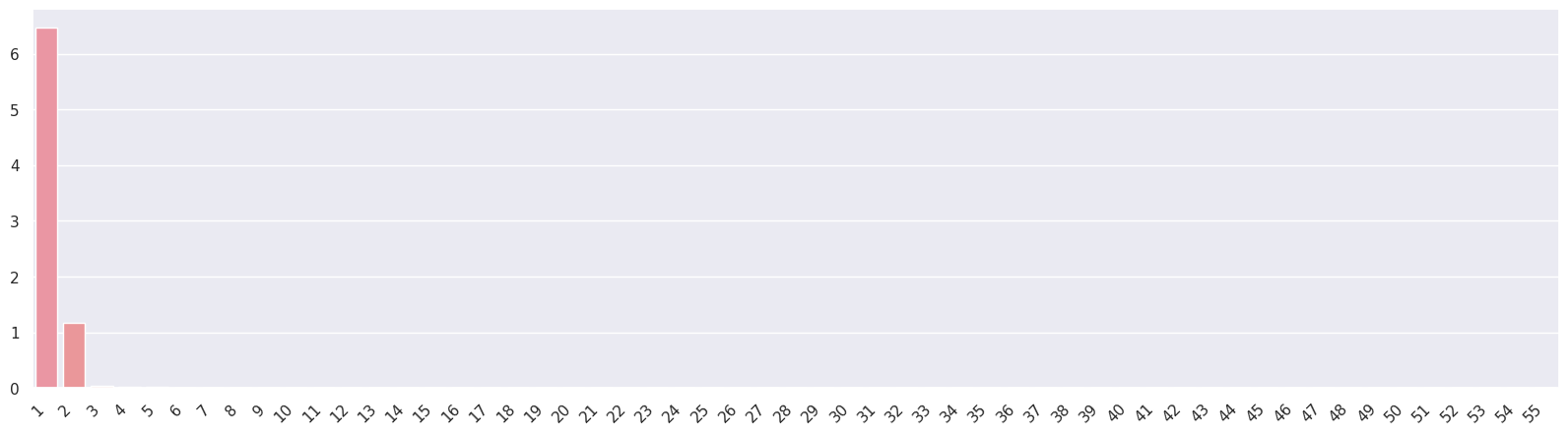}
        \caption{ENERGY}
    \end{subfigure}
    \begin{subfigure}[b]{\columnwidth}
        \includegraphics[width=\textwidth]{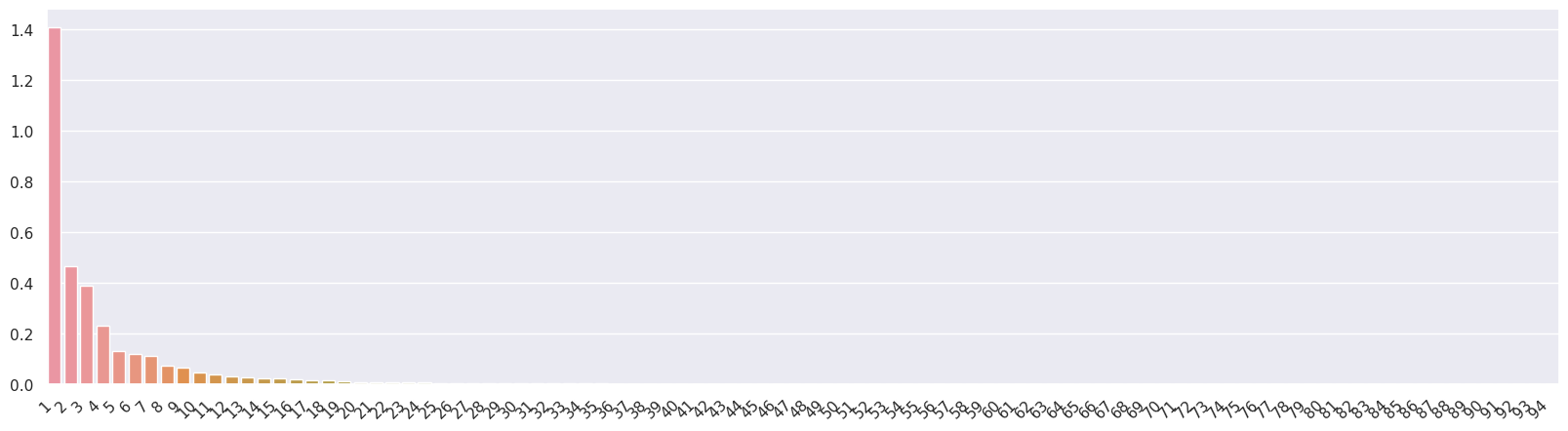}
        \caption{WINE}
    \end{subfigure}
    \begin{subfigure}[b]{\columnwidth}
        \includegraphics[width=\textwidth]{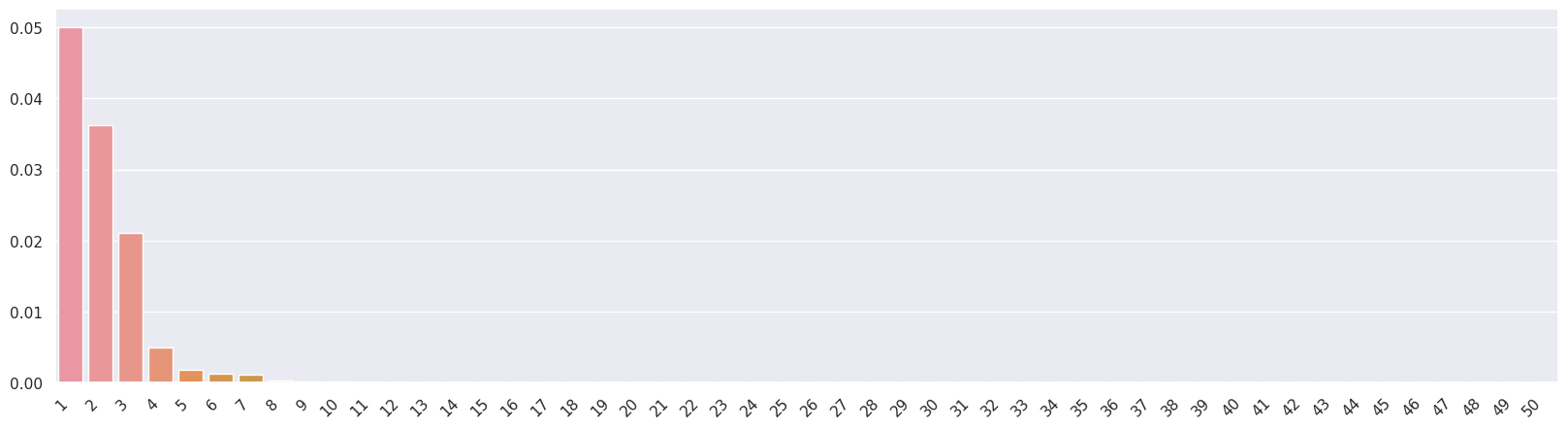}
        \caption{YACHT}
    \end{subfigure}
    \caption{The first 100 positive eigenvalues of $r(X_S,X)k(X,X_S)$ for datasets BOSTON, CONCRETE, ENERGY, WINE and YACHT. }
    \label{fig:eigenvalues}
    \vspace{-.3cm}
\end{figure}

We see in \Cref{fig:eigenvalues} that eigenvalues indeed decay fast.

\end{document}